%% file: main.tex
\documentclass[twoside]{article}

\usepackage[accepted]{conference2019}

\usepackage[numbers]{natbib}

\usepackage[utf8]{inputenc} % allow utf-8 input
\usepackage[T1]{fontenc}    % use 8-bit T1 fonts
\usepackage[hidelinks]{hyperref}       % hyperlinks % hidelinks
\usepackage{url}            % simple URL typesetting
\usepackage{booktabs}       % professional-quality tables
\usepackage{amsfonts}       % blackboard math symbols
\usepackage{nicefrac}       % compact symbols for 1/2, etc.
\usepackage{microtype}      % microtypography
\usepackage{subcaption}
\usepackage{graphicx}
\usepackage{xspace}
\input{shortcuts}

\def\regret{\mathrm{regret}}
\def\regretw{\mathrm{regret}^w}
\def\regretwp{\mathrm{regret}^w_{\mathrm{path}}}
\def\nceil{n_\mathrm{ceil}}
\def\sumnN{\sum_{n=1}^N}
\def\epsf{\epsilon^w_{{\hat \FF}}}
\def\epsp{\epsilon^w_{\Pi}}
\def\epsx{\epsilon^w_{\XX}}
\def\dagger{\textsc{DAgger}\xspace} 
 
\def\aggrevate{\textsc{AggreVaTe}\xspace}

\def\mobil{\textsc{MoBIL}\xspace}
\def\algVI{\textsc{MoBIL-VI}\xspace}  
\def\algprox{\textsc{MoBIL-Prox}\xspace}
\def\mirrorprox{\textsc{Mirror-Prox}\xspace}

\def\diam{{\rm{Diam}}}

\allowdisplaybreaks

\title{Accelerating Imitation Learning with Predictive Models}

\begin{document}

\twocolumn[

\conferencetitle{
	Accelerating Imitation Learning with Predictive Models
}

\conferenceauthor{ Ching-An Cheng \And Xinyan Yan \And  Evangelos A. Theodorou \And Byron Boots}

\conferenceaddress{ Georgia Tech \And  Georgia Tech \And Georgia Tech \And Georgia Tech } ]

\begin{abstract}
Sample efficiency is critical in solving real-world reinforcement learning problems, where agent-environment interactions can be costly. Imitation learning from expert advice has proved to be an effective strategy for reducing the number of interactions required to train a policy. Online imitation learning, which interleaves policy evaluation and policy optimization, is a particularly effective technique with provable performance guarantees. In this work, we seek to further accelerate the convergence rate of online imitation learning, thereby making it more sample efficient. We propose two model-based algorithms inspired by Follow-the-Leader (FTL) with prediction: \algVI based on solving variational inequalities and \algprox based on stochastic first-order updates. These two methods leverage a model 
to predict future gradients to speed up policy learning. When the model oracle is learned online, these algorithms can provably accelerate the best known convergence rate up to an order. Our algorithms can be viewed as a generalization of stochastic \mirrorprox (Juditsky et al., 2011), and admit a simple constructive FTL-style analysis of performance. 
\end{abstract}

\vspace{-1mm}
\section{INTRODUCTION}
\vspace{-2mm}
Imitation learning (IL) has recently received attention for its ability to speed up policy learning when solving reinforcement learning problems (RL)~\citep{abbeel2005exploration,ross2011reduction,ross2014reinforcement,chang2015learning,sun2017deeply,le2018hierarchical}. Unlike pure RL techniques, which rely on uniformed random exploration to locally improve a policy, IL leverages prior knowledge about a problem in terms of \emph{expert demonstrations}. At a high level, this additional information provides policy learning with
an informed search direction  toward the expert policy.

The  goal of IL is to quickly learn a policy that can perform at least as well as the expert policy. Because the expert policy may be suboptimal with respect to the RL problem of interest, performing IL 
is often used to provide a good warm start to the RL problem, 
so that the number of interactions with the environment can be minimized.
Sample efficiency is especially critical when learning is deployed in applications like robotics, where every interaction incurs real-world costs.

By reducing IL to an online learning problem, \emph{online IL}~\citep{ross2011reduction} provides a framework for convergence analysis and mitigates the covariate shift problem encountered in batch IL~\citep{argall2009survey,bojarski2017explaining}.
In particular, 
under proper assumptions, the performance of a policy sequence updated by Follow-the-Leader (FTL) can converge on average to the performance of the expert policy~\citep{ross2011reduction}. Recently, it was shown that this rate is sufficient to make IL more efficient than solving an RL problem from scratch~\citep{cheng2018fast}.

In this work, we further accelerate the convergence rate of online IL. Inspired by the observation of~\citet{cheng2018convergence} that the online learning problem of IL is \emph{not} truly adversarial, we propose two MOdel-Based IL (\mobil) algorithms, \algVI and \algprox, that can achieve a fast rate of convergence. Under the same assumptions of~\citet{ross2011reduction}, these algorithms improve on-average convergence to $O(1/N^2)$, e.g., when a dynamics model is learned online, where $N$ is the number of iterations of policy update.

The improved speed of our algorithms is attributed to using a model oracle to predict the gradient of the next per-round cost in online learning. This model can be realized, e.g., using a simulator based on a (learned) dynamics model, or using past demonstrations.
We first conceptually show that this idea can be realized as a variational inequality problem in \algVI. Next, we propose a practical first-order stochastic algorithm \algprox, which alternates between the steps of taking the true gradient and of taking the model gradient. \algprox is a generalization of stochastic \mirrorprox proposed by~\citet{juditsky2011solving} to the case where the problem is weighted and the vector field is unknown but learned online. 
In theory, we show that having a \emph{weighting} scheme is pivotal to speeding up convergence, and this generalization is made possible by a new constructive FTL-style regret analysis, which greatly simplifies the original algebraic proof~\citep{juditsky2011solving}. The performance of \algprox is also empirically validated in simulation.

\vspace{-1mm}
\section{PRELIMINARIES}
\vspace{-2mm}
\subsection{Problem Setup: RL and IL}
\vspace{-1mm}

Let $\Sbb$ and $\Abb$ be the state and the action spaces, respectively. 
The objective of RL is to search for a stationary policy $\pi$  inside a policy class $\Pi$ with good performance. This can be characterized by the stochastic optimization problem with  expected cost\footnote{Our definition of $J(\pi)$ corresponds to the average accumulated cost in the RL literature.} $J(\pi)$ defined below:
\begin{align}  \label{eq:RL problem}
\min_{\pi \in \Pi} J(\pi), \quad J(\pi) \coloneqq \E_{(s,t) \sim d_{\pi}}\E_{a \sim \pi_s} \left[ c_t(s, a) \right],
\end{align}
in which $s\in\Sbb$, $a\in\Abb$, $c_t$ is the instantaneous cost at time $t$, 
$d_{\pi}$ is a \emph{generalized stationary distribution} induced by executing policy $\pi$,
and $\pi_s$ is the distribution of action $a$ given state $s$ of $\pi$.
The policies here are assumed to be parametric. To make the writing compact, we will abuse the notation $\pi$ to also denote its parameter, and assume $\Pi$ is a compact convex subset of parameters in some normed space with norm $\norm{\cdot}$.

Based on the abstracted distribution $d_{\pi}$, the formulation in~\eqref{eq:RL problem} subsumes multiple discrete-time RL problems. 
For example, a $\gamma$-discounted infinite-horizon problem can be considered by setting $c_t = c$ as a time-invariant cost and defining the joint distribution $d_{\pi}(s,t) = (1-\gamma) \gamma^t d_{\pi,t}(s)$, in which $d_{\pi,t}(s)$ denotes the probability (density) of state $s$ at time $t$ under policy $\pi$. 
Similarly, a $T$-horizon RL problem can be considered by setting $d_{\pi}(s,t) = \frac{1}{T} d_{\pi,t}(s)$. 
Note that while we use the notation $\E_{a \sim \pi_s}$,  the policy is allowed to be deterministic; in this case, the notation means evaluation.
For notational compactness, we will often omit the random variable inside the expectation (e.g. we shorten \eqref{eq:RL problem} to $\E_{d_{\pi}}\E_{\pi} \left[ c \right]$).
In addition,  we denote $Q_{\pi,t}$ as the Q-function\footnote{For example, in a $T$-horizon problem, $ Q_{\pi,t}(s,a) = c_t(s,a) +  \E_{\rho_{\pi,t}}\left[{\scriptstyle\sum_{\tau=t}^{T-1}} \textstyle c_\tau (s_\tau, a_\tau)\right]$, where $\rho_{\pi,t}$ denotes the distribution of future trajectory $(s_t, a_t, s_{t+1}, \dots s_{T-1}, a_{T-1})$ conditioned on $s_t = s, a_t = a$.} at time $t$ with respect to $\pi$.

In this paper, we consider IL, which is an indirect approach to solving the RL problem. 
We assume there is a black-box oracle $\pi^*$, called the \emph{expert} policy, from which  demonstration $a^* \sim \pi_s^*$ can be queried for any state $s \in \Sbb$. To satisfy the querying requirement, usually the expert policy is an algorithm; for example, it can represent a planning algorithm which solves a simplified version of~\eqref{eq:RL problem}, or some engineered, hard-coded policy (see e.g.~\cite{pan2017agile}).

The purpose of incorporating the expert policy into solving~\eqref{eq:RL problem} 
is to quickly obtain a policy $\pi$ that has reasonable performance. 
Toward this end, we consider solving a surrogate problem of~\eqref{eq:RL problem}, 
\begin{align} \label{eq:IL problem}
\min_{\pi \in \Pi}  \E_{ (s, t) \sim d_{\pi}} [D(\pi_s^*||\pi_s)],
\end{align}
where $D$ is a function that measures the difference between two distributions over actions (e.g. KL divergence; see Appendix~\ref{app:choice of discussion}).
Importantly, the objective in~\eqref{eq:IL problem} has the property
that $D(\pi^*|| \pi^*) =0$ and there is constant $C_{\pi^*} \geq 0 $  such that
$ \forall t \in \Nbb, s \in \Sbb, \pi \in \Pi$, it satisfies
$\E_{a \sim \pi_s}[Q_{\pi^*,t}(s,a)]  - \E_{a^* \sim \pi_s^*}[Q_{\pi^*,t}(s,a^*)]
  \leq C_{\pi^*} D(\pi_s^*||\pi_s)$, 
in which $\Nbb$ denotes the set of natural numbers. 
By the Performance Difference Lemma~\citep{kakade2002approximately}, it can be shown that the inequality above implies~\citep{cheng2018convergence}, 
\begin{align} \label{eq:IL upper bound}
J(\pi) - J(\pi^*)   \leq  C_{\pi^*} \E_{ d_{\pi}} [D(\pi^*||\pi) ].
\end{align}
Therefore, solving~\eqref{eq:IL problem} can lead to a policy that performs similarly to the expert policy $\pi^*$. 

\vspace{-1mm}
\subsection{Imitation Learning as Online Learning}
\vspace{-2mm}

The surrogate problem in~\eqref{eq:IL problem} is  more structured than the original RL problem in~\eqref{eq:RL problem}. In particular, when the distance-like function $D$ is given,
and we know that $D(\pi^*||\pi)$ is close to zero when $\pi$ is close to $\pi^*$.
On the contrary, $\E_{a\sim \pi_s}[c_t(s,a)]$ in~\eqref{eq:RL problem} generally can still be large, even if $\pi$ is a good policy (since it also depends on the state). 
This \emph{normalization} property is crucial for the reduction from IL to online learning~\citep{cheng2018convergence}. 

The reduction is based on observing that, with the normalization property, the expressiveness of the policy class $\Pi$ can be described with a constant $\epsilon_{\Pi}$ defined as,
\begin{align} \label{eq:policy class complexity}
\textstyle
\hspace{-2mm}
 \epsilon_{\Pi} \geq \max_{\{ \pi_n \in \Pi \}}  \min_{\pi\in\Pi} \frac{1}{N} \sum_{n=1}^{N}  \E_{d_{\pi_n}}[D(\pi^*||\pi)],
\end{align}
for all $N \in \Nbb$,
which measures the average difference between $\Pi$ and $\pi^*$ with respect to $D$ and the state distributions visited by a worst possible policy sequence.
\citet{ross2011reduction} make use of this property and reduce~\eqref{eq:IL problem} into an online learning problem by distinguishing the influence of $\pi$ on $d_\pi$ and on $D(\pi^*||\pi)$ in~\eqref{eq:IL problem}. To make this transparent, we define a bivariate function 
\begin{align} \label{eq:bivariate function}
F(\pi', \pi) \coloneqq \E_{ d_{\pi'}} [D(\pi^*||\pi)].
\end{align}
Using this bivariate function $F$,
the online learning setup can be described as follows:
in round $n$, the learner applies a policy $\pi_n \in \Pi$ and then the environment reveals a per-round cost
\begin{align} \label{eq:policy per-round cost}
f_n(\pi) \coloneqq F(\pi_n, \pi) = \E_{d_{\pi_n}}[D(\pi^* || \pi)].
\end{align}
\citet{ross2011reduction} show that if the sequence $\{\pi_n\}$ is selected by a \emph{no-regret algorithm}, 
then it will have good performance in terms of~\eqref{eq:IL problem}.
For example,  
\dagger updates the policy by FTL,
$\pi_{n+1} = \argmin_{\pi \in \Pi} f_{1:n}(\pi)$ and has the following guarantee (cf. \citep{cheng2018convergence}), where we define the shorthand $f_{1:n} = \sum_{m=1}^{n} f_m$.
\begin{theorem} \label{th:dagger performance}
	Let $\mu_f >0$. 
	If each $f_n$ is $\mu_f$-strongly convex
	and 
	$\norm{\nabla f_n(\pi)}_* \le G, \forall \pi \in \Pi$,
	then 
	\dagger has 
	performance on average satisfying
	\begin{align} \label{eq:performance FTL}
	\textstyle \hspace{-2mm}
	\frac{1}{N}\sum_{n=1}^{N} J(\pi_n) \leq J(\pi^*) +  C_{\pi^*} \paren{\frac{G^2}{2\mu_f}\frac{\ln N + 1}{N}  + \epsilon_{\Pi}}.
	\end{align}
\end{theorem}\vspace{-2mm}
First-order variants of \dagger based on Follow-the-Regularized-Leader (FTRL)
have also been proposed by~\citet{sun2017deeply} and \citet{cheng2018fast}, which have the same performance but only require taking a stochastic gradient step in each iteration without keeping all the previous cost functions (i.e. data) as in the original FTL formulation.
The bound in Theorem~\ref{th:dagger performance} also applies to the expected performance of a policy randomly picked out of the sequence $\{\pi_n \}_{n=1}^N$, although it does not necessarily translate into the performance of the last policy $\pi_{N+1}$~\citep{cheng2018convergence}.

\vspace{-1mm}
\section{ACCELERATING IL WITH PREDICTIVE MODELS}
\vspace{-2mm}
\label{sec:fast IL}

The reduction-based approach to solving IL  has demonstrated sucess in speeding up policy learning. 
However, because interactions with the environment are necessary to approximately evaluate the per-round cost, it is interesting to determine if the convergence rate of IL can be further improved. A faster convergence rate  will be valuable in real-world applications where data collection is expensive. 

We answer this question affirmatively. We show that, by modeling\footnote{We define $\nabla_2 F$ as a vector field $\nabla_2 F : \pi \mapsto \nabla_2 F(\pi, \pi) $}   $\nabla_2 F$ the convergence rate of IL can potentially be improved by up to an order, where $\nabla_2$ denotes the derivative to the second argument. The improvement comes through leveraging the fact that the per-round cost $f_n$ defined in~\eqref{eq:policy per-round cost} is not completely unknown or adversarial as it is assumed in the most general online learning setting. Because the \emph{same} function $F$ is used in~\eqref{eq:policy per-round cost} over different rounds, the online component actually comes from the reduction made by~\citet{ross2011reduction}, which ignores information about how $F$ changes with the left argument; in other words, it omits the variations of $d_\pi$ when $\pi$ changes~\citep{cheng2018convergence}. Therefore, we argue that the original reduction proposed by~\citet{ross2011reduction}, while allowing the use of~\eqref{eq:policy class complexity} to characterize the performance, loses one critical piece of information present in the original RL problem: \emph{both the system dynamics and the expert are the same across different rounds of online learning.}

We propose two model-based algorithms (\algVI and \algprox) to accelerate IL.
The first algorithm, \algVI, is conceptual in nature and updates policies by solving variational inequality (VI) problems~\citep{facchinei2007finite}. 
This algorithm is used to illustrate how modeling $\nabla_2 F$
through a \emph{predictive model} $\nabla_2 \hat{F}$  can help to speed up IL, where $\hat{F}$ is a model bivariate function.\footnote{While we only concern predicting the vector field $\nabla_2 F$, we adopt the notation $\hat{F}$ to better build up the intuition, especially of \algVI; we will discuss other approximations that are not based on bivariate functions in Section~\ref{sec:model oracles}.}  
The second algorithm, \algprox is a first-order method. It alternates between taking stochastic gradients by interacting with the environment and querying the model $\nabla_2 \hat{F}$.
We will prove that this simple yet practical approach has the same performance as the conceptual one: when $\nabla_2 \hat{F}$ is learned online and $\nabla_2 F$ is realizable, e.g. both algorithms can converge in $O\paren{\frac{1}{N^2}}$, in contrast to \dagger's $O\paren{\frac{\ln N}{N}}$ convergence. In addition, we show the convergence results of \mobil under relaxed assumptions, e.g. allowing stochasticity, and 
provide several examples of constructing predictive models. (See Appendix~\ref{app:notation} for a summary of notation.)

\vspace{-1mm}
\subsection{Performance and Average Regret} \label{sec:reduction}
\vspace{-2mm}
Before presenting the two algorithms, we first summarize the core idea of the reduction from IL to online learning in a simple lemma, which builds the foundation of our algorithms (proved in Appendix~\ref{app:proof of reduction}). 
\begin{restatable}{lemma}{reductionLemma}
	\label{lm:reduction lemma}
	For arbitrary sequences $\{\pi_n \in \Pi \}_{n=1}^N$  and $\{w_n >0 \}_{n=1}^N$, it holds that
	\begin{align*}
	\textstyle
	 \E  \left[ \sum_{n=1}^{N} \frac{w_n J(\pi_n) }{w_{1:N}} \right] \leq J(\pi^*) + C_{\pi^*} \left( \epsilon_{\Pi}^w  + \E\left[ \frac{\regret^w(\Pi)}{w_{1:N}}\right]  \right)
	\end{align*} 
	where
	$\tilde{f}_n$ is an unbiased estimate of $f_n$,
	$\regret^w(\Pi)  \coloneqq \max_{\pi \in \Pi}   \sum_{n=1}^{N} w_n \tilde{f}_n (\pi_n) - w_n \tilde{f}_n (\pi) $, $\epsilon_{\Pi}^w$ is given in Definition~\ref{def:general class complexity}, 
	and the expectation is due to sampling $\tilde{f}_n$.
\end{restatable}
In other words, the on-average performance convergence of an online IL algorithm is determined by the rate of 
the expected  weighted average regret
$\E\left[\regret^w(\Pi) / w_{1:N} \right]$. For example, in \dagger, the weighting is uniform and $\E\left[\regret^w(\Pi) \right]$ is in $O(\log N)$; by Lemma~\ref{lm:reduction lemma} this rate directly proves Theorem~\ref{th:dagger performance}.

\vspace{-1mm}
\subsection{Algorithms}
\vspace{-2mm}

From Lemma~\ref{lm:reduction lemma}, we know that improving the regret bound implies a faster convergence of IL. 
This leads to the main idea of \algVI and \algprox: to use model information to \emph{approximately} play Be-the-Leader (BTL)~\citep{kalai2005efficient}, i.e. $\pi_{n+1} \approx \argmin_{\pi\in\Pi} f_{1:n+1}(\pi) $. To understand why playing BTL can minimize the regret, we recall a classical regret bound of online learning.\footnote{We use notation $x_n$ and $l_n$ to distinguish general online learning problems from online IL problems.}
\begin{restatable}[Strong FTL Lemma~\citep{mcmahan2017survey}]{lemma}{strongFTL} \label{lm:strong FTL}
	For any sequence of decisions $\{x_n \in \XX \}$ and loss functions $\{l_n\}$, 
	$	\regret(\XX)  \leq  \sum_{n=1}^{N}  l_{1:n} (x_n)- l_{1:n}( x_{n}^\star)$, where $
	x_n^\star \in \arg \min_{x \in \XX} l_{1:n}(x)
	$, where $\XX$ is the decision set.
\end{restatable}
Namely, if the decision $\pi_{n+1}$ made in round $n$ in IL is close to the best decision in round $n+1$ after the new per-round cost $f_{n+1}$ is revealed (which depends on $\pi_{n+1}$), then the regret will be small.

The two algorithms are summarized in Algorithm~\ref{alg:algs}, which mainly differ in the policy update rule  (line 5). 
Like \dagger, they both learn the policy in an interactive manner. In round $n$, both algorithms execute the current policy $\pi_n$ in the real environment to collect data to define the per-round cost functions (line 3): $\tilde{f}_n$ is an unbiased estimate of $f_n$ in~\eqref{eq:policy per-round cost} for policy learning, and $\tilde{h}_n$ is an unbiased estimate of the per-round cost $h_n$ for model learning. Given the current per-round costs, the two algorithms then update the model (line 4) and the policy (line 5) using the respective rules. 
Here we use the set $\hat{\FF}$, abstractly, to denote the family of predictive models to estimate $\nabla_2 F$, and $h_n$ is defined as an upper bound of the prediction error. For example, $\hat{\FF}$ can be a family of dynamics models that are used to simulate the predicted gradients,  and $\tilde{h}_n$ is the empirical loss function used to train the dynamics models (e.g. the KL divergence of prediction). 

\vspace{-1mm}
\subsubsection{A Conceptual Algorithm:  \algVI} \label{sec:conceptual algorithm}
\vspace{-2mm}

We first present our conceptual algorithm \algVI, which is simpler to explain. We assume that $f_n$ and $h_n$ are given, as in Theorem~\ref{th:dagger performance}. 
This assumption will be removed in \algprox later. 
To realize the idea of BTL, in round $n$, \algVI uses a newly learned predictive model $\nabla_2\hat{F}_{n+1}$ to estimate of $\nabla_2F$ in~\eqref{eq:bivariate function} and then 
updates the policy by solving the VI problem below: finding $\pi_{n+1} \in \Pi$ such that $\forall \pi' \in \Pi$,
\begin{align} \label{eq:VI problem}
\textstyle
\lr{\Phi_n(\pi_{n+1}) }{ \pi' - \pi_{n+1}} \geq 0,
\end{align}
where the vector field $\Phi_n$ is defined as 
\begin{align*}
\textstyle
\Phi_n(\pi) = \sum_{m=1}^{n} w_m \nabla f_{m}(\pi) + w_{n+1} \nabla_2 \hat {F}_{n+1}(\pi, \pi) 
\end{align*}

Suppose $\nabla_2 \hat{F}_{n+1}$ is the partial derivative of some bivariate function $\hat{F}_{n+1}$.
If $w_n =1$, then the VI problem\footnote{
	Because $\Pi$ is compact, the VI problem in~\eqref{eq:VI problem} has at least one solution~\citep{facchinei2007finite}. If $f_n$ is strongly convex,  the VI problem in line 6 of Algorithm~\ref{alg:algs} is strongly monotone for large enough $n$ and can be solved e.g. by basic projection method~\citep{facchinei2007finite}. Therefore, for demonstration purpose, we assume the VI problem of \algVI can be exactly solved.
} in~\eqref{eq:VI problem}
finds a fixed point $\pi_{n+1}$ satisfying
$\pi_{n+1} = \argmin_{\pi \in \Pi} f_{1:n}(\pi) + \hat F_{n+1}(\pi_{n+1}, \pi)$.
That is, if $\hat F_{n+1} = F$ exactly, then $\pi_{n+1}$ plays exactly BTL and by  Lemma~\ref{lm:strong FTL} the regret is non-positive. 
In general, we can show that, even with modeling errors, \algVI can still reach a faster convergence rate such as  $O\paren{\frac{1}{N^2}}$, if a non-uniform weighting scheme is used, the model is updated online, and $\nabla_2 F$ is realizable within $\hat{\FF}$.
The details will be presented in Section~\ref{sec:conceptual algorithm (analysis)}.

\begin{algorithm}[t]
	{\small
		\caption{\mobil }\label{alg:algs} 
		\begin{algorithmic} [1]
			\renewcommand{\algorithmicensure}{\textbf{Input:}}			
			\renewcommand{\algorithmicrequire}{\textbf{Output:}}
			\ENSURE  $\pi_1$, $N$, $p$
			\REQUIRE $\bar \pi_N$
			\STATE Set weights $w_n = n^p$ for $n = 1, \dots, N$ 
			and sample integer $K$ with $P(K=n) \propto w_n$ 
			\FOR {$n = 1\dots K-1$}
			\STATE Run $\pi_{n}$ in the real environment to collect data to define $\tilde f_n$ and $\tilde{h}_n$\footnotemark
			\STATE  Update the predictive model to $\nabla_2 \hat{F}_{n+1}$; e.g., using FTL
			$\hat F_{n+1} = \argmin_{\hat{F} \in \hat{\FF} } \sum_{m=1}^n \frac{w_m}{m} \tilde h_m(\hat{F} )$
			\STATE Update policy to $\pi_{n+1}$ by~\eqref{eq:VI problem} (\algVI) or by~\eqref{eq:practical policy update} (\algprox)
			\ENDFOR
			\STATE Set $\bar \pi_N = \pi_K$
		\end{algorithmic}
	}
\end{algorithm}
\footnotetext{\algVI assumes $\tilde f_n  = f_n$ and $\tilde{h}_n = h_n$}

\vspace{-1mm}
\subsubsection{A Practical Algorithm:  \algprox} \label{sec:practical algorithm}
\vspace{-2mm}

While the previous conceptual algorithm achieves a faster convergence, it requires solving a nontrivial VI problem in each iteration. 
In addition, it assumes $f_n$ is given as a function 
and requires keeping all the past data to define $f_{1:n}$. Here we relax these unrealistic assumptions and propose \algprox. 
In round $n$ of \algprox, 
the policy is updated from $\pi_n$ to $\pi_{n+1}$ by 
\emph{taking two gradient steps}:
\newenvironment{talign}
{\let\displaystyle\textstyle\align}
{\endalign}
\begin{talign} \label{eq:practical policy update}
	\hspace{-3mm}
	\begin{split}
		\hat {\pi}_{n+1} &= \argmin_{\pi \in \Pi} \sum_{m=1}^n w_m\big(\lr{g_m}{\pi} + r_m(\pi) \big),\\
		 \pi_{n+1} &= \argmin_{\pi \in \Pi} 
		 w_{n+1} \lr{\hat g_{n+1}}{\pi} + \\ &\hspace{22mm} \sum_{m=1}^n w_m\big(\lr{g_m}{\pi} + r_m(\pi) \big) 
	\end{split}
\end{talign}
We define $r_n$ as an $\alpha_n \mu_f$-strongly convex function (with $\alpha_n \in (0, 1]$; we recall $\mu_f$ is the strongly convexity modulus of $f_n$) such that $\pi_n$ is its global minimum and $r_n(\pi_n)=0$ (e.g. a Bregman divergence). And we define $g_n$ and $\hat g_{n+1}$ as estimates of $\nabla f_n(\pi_n) = \nabla_2 F(\pi_n, \pi_n)$  and  $\nabla_2 \hat F_{n+1}(\hat \pi_{n+1}, \hat \pi_{n+1})$, respectively.  Here we only require $g_n = \nabla \tilde f_n(\pi_n)$ to be unbiased, whereas $\hat{g}_n$ could be a biased estimate of  $\nabla_2 \hat F_{n+1}(\hat \pi_{n+1}, \hat \pi_{n+1})$.

\algprox treats $\hat{\pi}_{n+1}$, which plays FTL with $g_n$ from the real environment, as a rough estimate of the next policy $\pi_{n+1}$ 
and uses it to query an gradient estimate $\hat{g}_{n+1}$ from the model $\nabla_2 \hat{F}_{n+1}$. Therefore, the learner's decision $\pi_{n+1}$ can approximately play BTL.
If we compare the update rule of $\pi_{n+1}$ and the VI problem in~\eqref{eq:VI problem}, we can see that \algprox linearizes the problem and attempts to approximate $\nabla_2 \hat{F}_{n+1}(\pi_{n+1}, \pi_{n+1})$ by $\hat{g}_{n+1}$. 
While the above approximation is crude, interestingly it is sufficient to speed up the convergence rate to be as fast as \algVI  under mild assumptions, as shown later in Section~\ref{sec:practical algorithm (analysis)}.

\vspace{-1mm}
\subsection{Predictive Models} \label{sec:model oracles}
\vspace{-2mm}

\mobil uses  $\nabla_2 \hat{F}_{n+1}$ in the update rules~\eqref{eq:VI problem} and~\eqref{eq:practical policy update} at round $n$ to predict the unseen gradient at round $n+1$ for speeding up policy learning. Ideally $\hat{F}_{n+1}$ should approximate the unknown bivariate function $F$ so that $\nabla_2 F$ and $\nabla_2\hat{F}_{n+1}$ are close. 
This condition can be seen from~\eqref{eq:VI problem} and~\eqref{eq:practical policy update}, in which \mobil concerns only $\nabla_2\hat{F}_{n+1}$ instead of $\hat{F}_{n+1}$ directly.
In other words, $\nabla_2 \hat{F}_{n+1}$ is used in \mobil as a first-order oracle, which leverages all the past information (up to the learner playing $\pi_n$ in the environment at round $n$) to predict the future gradient $\nabla_2 {F}_{n+1}(\pi_{n+1},\pi_{n+1})$, which depends on the decision $\pi_{n+1}$ the learner is about to make.
Hence, we call it a predictive model.

To make the idea concrete, we provide a few examples of these models. 
By definition of $F$ in~\eqref{eq:bivariate function}, one way to construct the predictive model $\nabla_2 \hat{F}_{n+1}$ is through a \emph{simulator} with an (online learned) dynamics model, and define $\nabla_2 \hat{F}_{n+1}$ as the simulated gradient (computed by querying the expert along the simulated trajectories visited by the learner). If the dynamics model is exact, then $\nabla_2 \hat{F}_{n+1} = \nabla_2 F$. Note that a stochastic/biased estimate of $\nabla_2 \hat{F}_{n+1}$ suffices to update the policies in \algprox. 

Another idea is to construct the predictive model through $\tilde{f}_n$ (the stochastic estimate of $f_n$) and indirectly define $\hat{F}_{n+1}$ such that $\nabla_2 \hat{F}_{n+1} = \nabla \tilde{f}_n$. This choice is possible, because the learner in IL collects \emph{samples} from the environment, as opposed to, literally, gradients. Specifically, we can define $g_n = \nabla \tilde{f}_{n}(\pi_n)$ and $\hat{g}_{n+1} = \nabla \tilde{f}_n(\hat{\pi}_{n+1})$ in~\eqref{eq:practical policy update}. 
The approximation error of setting $\hat{g}_{n+1} = \nabla \tilde{f}_n(\hat{\pi}_{n+1})$ is determined by the convergence and the stability of the learner's policy. If $\pi_n$ visits similar states as $\hat{\pi}_{n+1}$, then $\nabla \tilde{f}_n$ can approximate $\nabla_2 F$ well at $\hat{\pi}_{n+1}$. Note that this choice is different from using the previous gradient (i.e. $\hat{g}_{n+1} = g_n$) in optimistic mirror descent/FTL~\citep{rakhlin2013online}, which would have a larger approximation error due to additional linearization. 

Finally, we note that while the concept of predictive models originates from estimating the partial derivatives $\nabla_2 F$, a predictive model does not necessarily have to be in the same form. A parameterized vector-valued function can also be directly learned 
  to approximate $\nabla_2 F$, e.g., using a neural network and the sampled gradients $\{g_n\}$ in a supervised learning fashion.

\vspace{-1mm}
\section{THEORETICAL ANALYSIS} \label{sec:analysis}
\vspace{-2mm}

Now we prove that using predictive models in \mobil can  accelerate convergence, when proper conditions are met. Intuitively, \mobil converges faster than the usual adversarial approach to IL (like \dagger), when the predictive models have smaller errors than not predicting anything at all (i.e. setting $\hat{g}_{n+1}=0$). 
In the following analyses, we will focus on bounding the expected weighted average regret, as it directly translates into the average performance bound by Lemma~\ref{lm:reduction lemma}. We define,  for  $w_n = n^p$, 
\begin{align} \label{eq:RR}
\RR(p) \coloneqq \E\left[\regret^w(\Pi) / w_{1:N} \right]
\end{align}Note that the results below assume that the predictive models are updated using FTL as outlined in Algorithm~\ref{alg:algs}. This assumption applies, e.g., when a dynamics model is learned online in a simulator-oracle as discussed above.
We provide full proofs in Appendix~\ref{app:proofs of analysis} and provide a summary of notation in Appendix~\ref{app:notation}.

\vspace{-1mm}
\subsection{Assumptions} \label{sec:assumptions}
\vspace{-2mm}
We first introduce several assumptions to more precisely characterize the online IL problem. 

\vspace{-1mm}
\paragraph{Predictive models}
Let $\hat{\FF}$ be the class of predictive models.
We assume these models are  Lipschitz continuous in the following sense.
\begin{assumption} \label{as:Lipschitz continuity}
	There is $L \in [0, \infty)$ such that
	$
	\norm{\nabla_2 \hat {F}(\pi,\pi) - \nabla_2 \hat {F}(\pi',\pi')}_* \leq L \norm{\pi - \pi'} 
	$,  $\forall \hat F \in \hat \FF$ and $\forall \pi, \pi' \in \Pi$. 
\end{assumption}\vspace{-2mm}

\vspace{-1mm}
\paragraph{Per-round costs} 
The per-round cost $f_n$ for policy learning is given in~\eqref{eq:policy per-round cost}, and we define $h_n(\hat{F})$ as an upper bound of $\norm{\nabla_2 {F} (\pi_{n}, \pi_{n}) - \nabla_2 \hat F(\pi_{n}, \pi_{n})  }_*^2$ (see e.g. Appendix~\ref{app:learning dynamics}).
We make structural assumptions on $\tilde{f}_n$ and $\tilde{h}_n$, similar to the ones made by~\citet{ross2011reduction} (cf. Theorem~\ref{th:dagger performance}).
\begin{assumption} \label{as:function structure}
	Let $\mu_f,\mu_h>0$.
	With probability $1$, $\tilde{f}_n$ is $\mu_f$-strongly convex, and $\norm{\nabla \tilde{f}_n(\pi)}_* \le G_f$, $\forall \pi \in \Pi$;  $\tilde{h}_n$ is $\mu_h$-strongly convex, and $\norm{\nabla \tilde{h}_n(\hat{F})}_* \le G_h$, $\forall \hat{F} \in \hat{\FF}$.	
\end{assumption}\vspace{-2mm}
By definition, these properties extend to $f_n$ and $h_n$.
We note they can be relaxed to solely \emph{convexity} and our algorithms still improve the best known convergence rate (see Table~\ref{table:convex cases} and  Appendix~\ref{app:convex cases}).

\begin{table*}[h]
	\small
	\caption[cap]{Convergence Rate Comparison\footnotemark}
	\label{table:convex cases}
	\centering
	\begin{tabular}{lccc}		
		\toprule
		\cmidrule(r){2-3}
		&$\tilde h_n$ convex    & $\tilde h_n$ strongly convex     & Without model \\
		\midrule
		$\tilde f_n$ convex &$O(N^{-3/4})$ & $O(N^{-1})$  & $O(N^{-1/2})$     \\
		$\tilde f_n$ strongly convex &$O(N^{-3/2})$    & $O(N^{-2})$ & $O(N^{-1})$     \\
		\bottomrule
	\end{tabular}	
	\vspace{-3mm}
\end{table*}
\footnotetext{The rates here assume $\sigma_{\hat g}, \sigma_{g}, \epsilon_{{\hat \FF}}^w = 0$. In general, the rate of \algprox becomes the improved rate in the table plus the ordinary rate multiplied by $C = \sigma_{g}^2 +\sigma_{\hat g}^2 + \epsilon_{{\hat \FF}}^w$. For example, when $\tilde{f}$ is convex and $\tilde{h}$ is strongly convex, \algprox converges in $O(1/N + C/\sqrt{N})$, whereas \dagger converges in $O(G_f^2 / \sqrt{N})$.}

\vspace{-1mm}
\paragraph{Expressiveness of hypothesis classes}
We introduce two constants, $\epsilon_{\Pi}^w$ and $\epsilon_{{\hat \FF}}^w$, to characterize the policy class $\Pi$ and model class ${\hat \FF}$, which generalize the idea of~\eqref{eq:policy class complexity} to stochastic and general weighting settings. 
When  $\tilde{f}_n = f_n$ and $\theta_n$ is constant, Definition~\ref{def:general class complexity} agrees with~\eqref{eq:policy class complexity}.
Similarly, we see that if 
$\pi^* \in \Pi$
and $F \in {\hat \FF}$, then $\epsilon_{\Pi}^w$ and  $\epsilon_{{\hat \FF}}^w$ are zero.
\begin{definition} \label{def:general class complexity}
	A policy class $\Pi$ is \emph{$\epsilon_{\Pi}^w$-close to $\pi^*$}, if for all $N \in \N$ and weight sequence $\{ \theta_n >0\}_{n=1}^N$ with $\theta_{1:N} = 1$, 
	$\expectb{ \max_{\{ \pi_n \in \Pi \}} \min_{\pi\in\Pi}  \sum_{n=1}^{N}  \theta_n \tilde{f}_n(\pi)  } \leq \epsilon_{\Pi}^w
	$. 
	Similarly, a model class ${\hat \FF}$ is \emph{$\epsilon_{{\hat \FF}}^w$-close to $F$}, if $\expectb{\max_{\{ \pi_n \in \Pi \}} \min_{\hat {F}\in {\hat \FF}}  \sum_{n=1}^{N}  \theta_n \tilde{h}_n(\hat {F})}  \leq \epsilon_{{\hat \FF}}^w $. The expectations above are due to sampling $\tilde{f}_n$ and $\tilde{h}_n$.
\end{definition}

\vspace{-1mm}
\subsection{Performance of \algVI}  \label{sec:conceptual algorithm (analysis)}
\vspace{-2mm}

Here we show the performance for \algVI when there is prediction error in $\nabla_2 \hat {F}_n$. The main idea is to treat \algVI as online learning with prediction~\citep{rakhlin2013online} and take
$ \hat {F}_{n+1} (\pi_{n+1}, \cdot)$ obtained after solving the VI problem~\eqref{eq:VI problem} as an \emph{estimate} of $f_{n+1}$.
\begin{restatable}{proposition}{constantRegretOfConceptualAlgorithm} \label{pr:constant regret of conceptual algorithm}
	For \algVI with $p = 0$,
	$\RR(0) \le \frac{G_f^2}{2\mu_f\mu_h}\frac{1}{N} + \frac{\epsf}{2 \mu_f}\frac{\ln N + 1}{N}$.
\end{restatable}

By Lemma~\ref{lm:reduction lemma}, this means that if the model class is expressive enough (i.e $\epsilon^w_{{\hat \FF}} = 0$), then by adapting the model online with FTL, we can improve the original convergence rate in $O(\ln N/N)$ of \citet{ross2011reduction} to 
$O(1/N)$.
While removing the $\ln N$ factor does not seem like much, we will show that running \algVI can improve the convergence rate to $O(1/N^2)$, when a \emph{non-uniform} weighting is adopted.
\begin{restatable}{theorem}{weightedPerformanceOfConceptualAlgorithm} \label{th:weighted performance of conceptual algorithm}
	For \algVI with $p>1$, 
	$ R(p) \leq C_p \left(  \frac{ pG_h^2}{2 (p-1)\mu_h}  \frac{1}{N^2}+ 
	\frac{\epsf}{pN}
	\right)$,
	where $ C_p = \frac{(p+1)^2 e^{p/N}}{2 \mu_f}$.
\end{restatable}
The key  is that
$\regret^w(\Pi)$ can be upper bounded by the regret of the online learning for models, which has per-round cost $\frac{w_n}{n}h_n$. 
Therefore, 
if $\epsf \approx 0$, randomly picking a policy out of $\{\pi_n \}_{n=1}^N$ proportional to weights $\{ w_n \}_{n=1}^N$ has expected convergence in $O\paren{\frac{1}{N^2}}$ if $p>1$.\footnote{If $p=1$, it converges in  $O\paren{\frac{\ln N}{N^2}}$; if $p \in [0,1)$, it converges in $O\paren{\frac{1}{N^{1+p}}}$. See Appendix~\ref{app:proof of conceptual algorithm}.}

\vspace{-1mm}
\subsection{Performance of \algprox} \label{sec:practical algorithm (analysis)}
\vspace{-2mm}

As \algprox uses gradient estimates, we additionally define two constants $\sigma_g$ and $\sigma_{{\hat g}}$ to characterize the estimation error, where $\sigma_{\hat{g}}$ also entails potential bias.
\begin{assumption} \label{as:variance}
	$\E[\norm{g_n - \nabla_2 F(\pi_n, \pi_n) }_*^2] \leq \sigma_g^2$ 
	and $\E[\norm{\hat g_{n} - \nabla_2 \hat F_{n}(\hat{\pi}_{n}, \hat{\pi}_{n}) }_*^2] \leq \sigma_{\hat g}^2$
\end{assumption}

We show  this simple first-order algorithm achieves similar performance to \algVI. Toward this end, we introduce a stronger lemma than Lemma~\ref{lm:strong FTL}.
\begin{restatable}[Stronger FTL Lemma]{lemma}{strongerFTL} \label{lm:stronger FTL}
	Let $
	x_n^\star \in \arg \min_{x \in \XX} l_{1:n}(x)
	$. 	For any sequence of decisions $\{x_n\}$ and losses $\{l_n\}$,
	$	\regret(\XX)  =  \sum_{n=1}^{N} l_{1:n}(x_n)- l_{1:n}( x_{n}^\star) - \Delta_{n} $, where
	$\Delta_{n+1} := l_{1:n}(x_{n+1}) - l_{1:n}(x_{n}^\star) \geq 0$. 
\end{restatable}
The additional $-\Delta_n$ term in Lemma~\ref{lm:stronger FTL} is pivotal to prove the performance of \algprox.
\begin{restatable}{theorem}{weightedPerformanceOfPracticalAlgorithm} \label{th:weighted performance of practical algorithm}
	For \algprox with $p >1$ and $\alpha_n = \alpha \in (0,1]$, it satisfies
	\begin{align*}
	\textstyle
	\RR(p)
	\le \frac{ (p+1)^2 e^{\frac{p}{N}}}{\alpha \mu_f} \paren{\frac{G_h^2}{\mu_h} \frac{p}{p-1} \frac{1}{N^2} + 
		\frac{2}{p} \frac{\sigma_g^2 + \sigma_{\hat g}^2 + \epsf}{N}} + \frac{(p+1) \nu_p}{N^{p+1}},
	\end{align*}
	where 
	$\nu_p =O(1)$
	and $\nceil = \ceil{\frac{2e^{\frac{1}{2}}(p+1)LG_f}{\alpha \mu_f}}$.	
\end{restatable}\vspace{-2mm}
\begin{proof}[Proof sketch]
	Here we give a proof sketch in big-O notation 
	(see Appendix~\ref{app:proof of practical algorithm} for the details). 
	To bound $\RR(p)$, recall the definition $\regret^w(\Pi) = \sum_{n=1}^{N} w_n \tilde{f}_n (\pi_n) - \min_{\pi \in \Pi}  \sum_{n=1}^{N} w_n \tilde{f}_n (\pi) $. 
	Now define $\bar f_n(\pi) := \lr{g_n}{\pi} + r_n(\pi)$. 
	Since $\tilde f_n$ is $ \mu_f$-strongly convex, $r_n$ is $ \alpha \mu_f$-strongly convex, and $r(\pi_n)=0$, we know that $\bar f_n$ satisfies that
	$\tilde f_n(\pi_n) - \tilde f_n(\pi) \le \bar f_n(\pi_n) - \bar f_n(\pi)$, $\forall \pi \in \Pi$. This implies $\RR(p) \leq \E[  \regretwp(\Pi) /w_{1:N} ]$, where $\regretwp(\Pi) \coloneqq \sum_{n=1}^N w_n \bar f_n(\pi_n) - \min_{\pi \in \Pi} \sum_{n=1}^N w_n \bar f_n(\pi)$. 
	
	The following lemma upper bounds $\regretwp(\Pi)$ by using Stronger FTL lemma (Lemma~\ref{lm:stronger FTL}).
	\begin{restatable}{lemma}{pathwiseRegretBound} \label{lm:pathwise regret bound}
		$\regretwp (\Pi) \le 
		\frac{p+1}{2 \alpha\mu_f}  \sum_{n=1}^N n^{p-1} \norm{g_n - \hat {g}_{n}}_*^2 -
		\frac{\alpha \mu_f }{2(p+1)} \sum_{n=1}^{N} (n-1)^{p+1} \norm{\pi_{n} - \hat {\pi}_{n}}^2$.
	\end{restatable}
	Since the second term in Lemma~\ref{lm:pathwise regret bound} is negative, 
	we just need to upper bound the expectation of the first item.  
	Using the triangle inequality, we bound the model's prediction error of the next per-round cost.
	\begin{restatable}{lemma}{expectedQuadraticError} \label{lm:expected quadratic error}
		$\E[\norm{g_{n} - \hat {g}_{n}}^2_* ] \leq 
		4 (\sigma_g^2 + \sigma_{\hat {g}}^2 + 
		L^2 \E[\norm{\pi_{n} - \hat {\pi}_{n}}^2]
		+ \E[\tilde h_{n}(\hat {F}_{n})]).
		$
	\end{restatable}	
	With Lemma~\ref{lm:expected quadratic error} and Lemma~\ref{lm:pathwise regret bound}, it is now clear that 
	$ 
	\E[\regretwp(\Pi)] \le  
	\E[\sum_{n=1}^{N} \rho_n \norm{\pi_n - \hat{\pi}_n }^2]+ 
	O(N^p) (\sigma_g^2 + \sigma_{\hat g}^2) + 
	O(  \E[\sum_{n=1}^{N} n^{p-1}\tilde h_n(\hat F_n)] )
	$, where 
	$\rho_n =O(n^{p-1} - n^{p+1})$.
	When $n$ is large enough, $\rho_n \le 0$, and hence the first term is $O(1)$.
	For the third term, because the model is learned online using, e.g., FTL with strongly convex cost $n^{p-1} \tilde{h}_n$ 
	we can show that
	$\E[\sum_{n=1}^{N} n^{p-1} \tilde h_n(\hat F_n)]= O ( N^{p-1} + N^p \epsf)$.
	Thus, $\E[ \regretwp (\Pi) ] \leq O( 1 + N^{p-1}+  (\epsilon_{{\hat \FF}}^w + \sigma_{g}^2 + \sigma_{\hat g}^2 ) N^{p}) $. 
	Substituting this bound into $\RR(p) \leq \E[  \regretwp(\Pi)  /w_{1:N}  ]$ and using that the fact $w_{1:N} = \Omega(N^{p+1})$ proves the theorem.
\end{proof}\vspace{-2mm}
The main assumption in Theorem~\ref{th:weighted performance of practical algorithm} is  that $\nabla_2 \hat {F}$ is $L$-Lipschitz continuous (Assumption~\ref{as:Lipschitz continuity}). It does not depend on the continuity of $\nabla_2 F$. Therefore, this condition is practical as we are free to choose  ${\hat \FF}$. 
Compared with Theorem~\ref{th:weighted performance of conceptual algorithm}, Theorem~\ref{th:weighted performance of practical algorithm} considers the inexactness of $\tilde{f}_n$ and $\tilde{h}_n$ explicitly; hence the additional term due to $\sigma_g^2$ and $\sigma_{{\hat g}}^2$.
Under the same assumption of \algVI that $f_n$ and $h_n$ are directly available, we can actually show that the simple \algprox has the same performance as \algVI, which is a corollary of Theorem~\ref{th:weighted performance of practical algorithm}.

\begin{corollary}
	\label{co:weighted performance of practical algorithm (deterministic)}
	If $\tilde{f}_n = f_n $ and $\tilde{h}_n = h_n$, for \algprox with $p > 1$,
	$\RR(p) \leq O( \frac{1}{N^2} + \frac{\epsilon_{{\hat \FF}}^w}{N}) $.
\end{corollary}\vspace{-2mm}

The proof of Theorem~\ref{th:weighted performance of conceptual algorithm} and~\ref{th:weighted performance of practical algorithm} are based on assuming the predictive models are updated by FTL (see Appendix~\ref{app:learning dynamics} for a specific bound when online learned dynamics models are used as a simulator). However, we note that these results are essentially based on the property that model learning also has no regret; therefore, the FTL update rule (line 4)
can be replaced by a no-regret first-order method without changing the result. This would make the algorithm even simpler to implement. 
The convergence of other types of predictive models (like using the previous cost function discussed in Section~\ref{sec:model oracles}) can also be analyzed following the major steps in the proof of Theorem~\ref{th:weighted performance of practical algorithm}, leading to a performance bound in terms of prediction errors.
Finally, it is interesting to note that the accelerated convergence is made possible when model learning puts more weight on costs in later rounds (because $p>1$). 

\vspace{-1mm}
\subsection{Comparison}
\vspace{-2mm}

We compare the performance of \mobil in Theorem~\ref{th:weighted performance of practical algorithm} with that of \dagger in Theorem~\ref{th:dagger performance} in terms of the constant on the $\frac{1}{N}$ factor.
\mobil has a constant in $O(\sigma_g^2 + \sigma_{\hat g}^2 + \epsf)$, whereas \dagger has a constant in $G_f^2 = O(G^2 + \sigma_g^2)$, where we recall $G_f$ and $G$ are upper bounds of $\norm{\nabla \tilde{f}_n(\pi)}_*$ and $\norm{\nabla f_n(\pi)}_*$, respectively.\footnote{Theorem~\ref{th:dagger performance} was stated by assuming $f_n=\tilde{f}_n$. In the stochastic setup here, \dagger has a similar convergence rate in expectation but with $G$ replaced by $G_f$.} Therefore, in general, \algprox has a better upper bound than \dagger when the model class is expressive (i.e. $\epsilon_{{\hat \FF}} \approx 0$), because $\sigma_{{\hat g}}^2$ (the variance of the sampled gradients) can be made small as we are free to design the model.
Note that, however, the improvement of \mobil may be smaller when the problem is noisy, such that the large $\sigma_g^2$ becomes the dominant term.

An interesting property that arises from Theorems~\ref{th:weighted performance of conceptual algorithm} and~\ref{th:weighted performance of practical algorithm}  is that the convergence of \mobil  is not biased by using an imperfect model (i.e. $\epsf>0$). This is shown in the term $\epsf/N$. In other words, in the worst case of using an extremely wrong predictive model, \mobil would just converge more slowly but still to the performance of the expert policy.

\algprox is closely related to stochastic Mirror-Prox~\citep{nemirovski2004prox,juditsky2011solving}. In particular, when the exact model is known (i.e. $\nabla_2\hat {F}_n = \nabla_2F$) and \algprox is set to convex-mode (i.e.  $r_n=0$ for $n>1$, and $w_n = 1/\sqrt{n}$; see Appendix~\ref{app:convex cases}), then \algprox gives the same update rule as stochastic Mirror-Prox with step size $O(1/\sqrt{n})$ (See Appendix~\ref{app:stochastic mirror-prox} for a thorough discussion).
Therefore, \algprox can be viewed as a generalization of Mirror-Prox: 1) it allows non-uniform weights; and 2) it allows the vector field $\nabla_2 F$ to be estimated online by alternately taking stochastic gradients and predicted gradients. 
The design of \algprox is made possible by our Stronger FTL lemma (Lemma~\ref{lm:stronger FTL}), which greatly simplifies the original algebraic proof in~\citep{nemirovski2004prox,juditsky2011solving}. 
Using Lemma~\ref{lm:stronger FTL} reveals more closely the interactions between model updates and policy updates. In addition, it more clearly shows the effect of non-uniform weighting, which is essential to achieving $O(\frac{1}{N^2})$ convergence. 
To the best of our knowledge, even the analysis of the original (stochastic) Mirror-Prox from the FTL perspective is new.

\begin{figure*} 	
	\centering
	\begin{subfigure}{.24\textwidth}
		\includegraphics[trim={1.0cm 1.15cm 1.2cm 0cm}, clip, width=\textwidth]{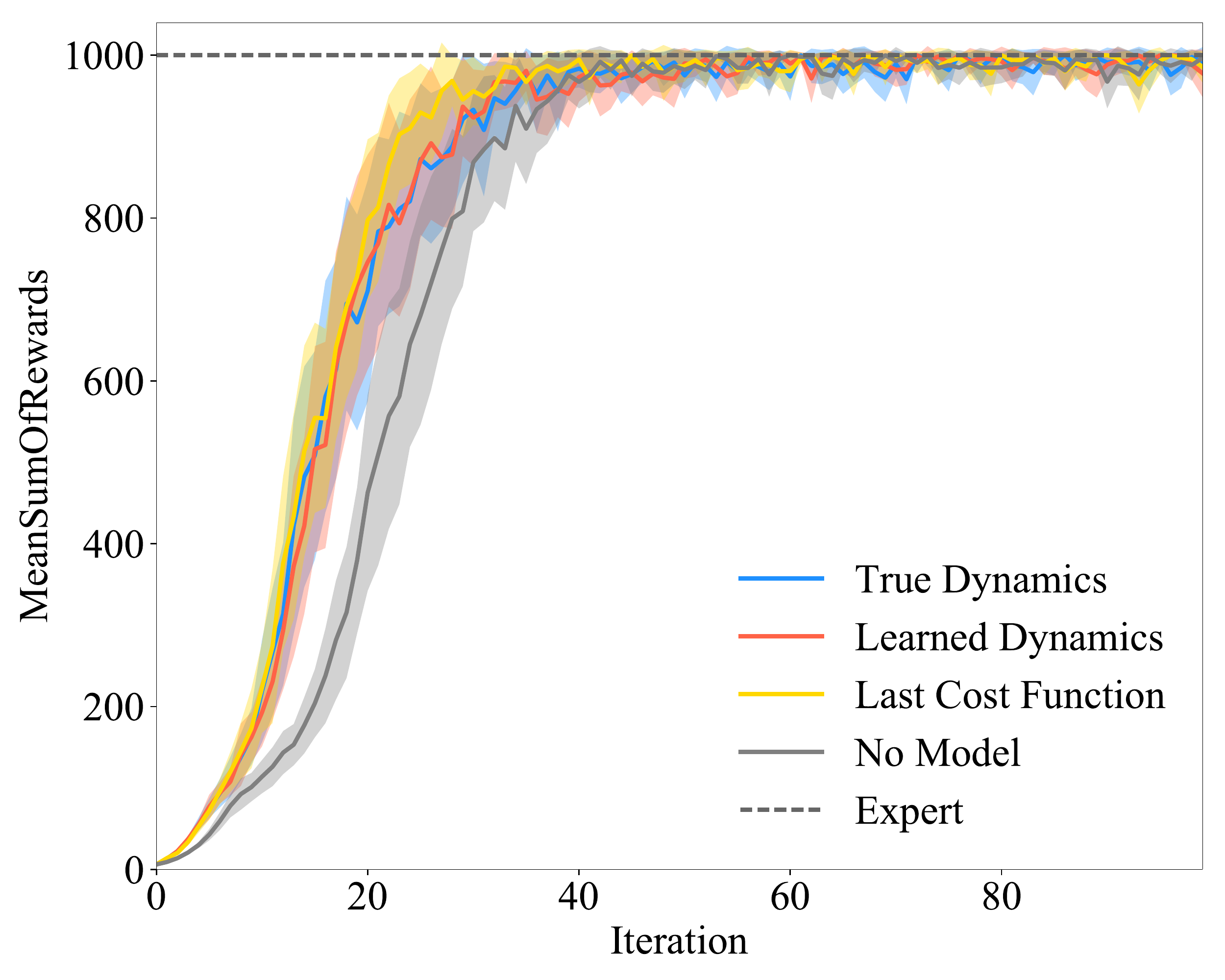}
	\end{subfigure}
	\begin{subfigure}{.24\textwidth}
		\includegraphics[trim={1.0cm 1.15cm 1.2cm 0cm}, clip, width=\textwidth]{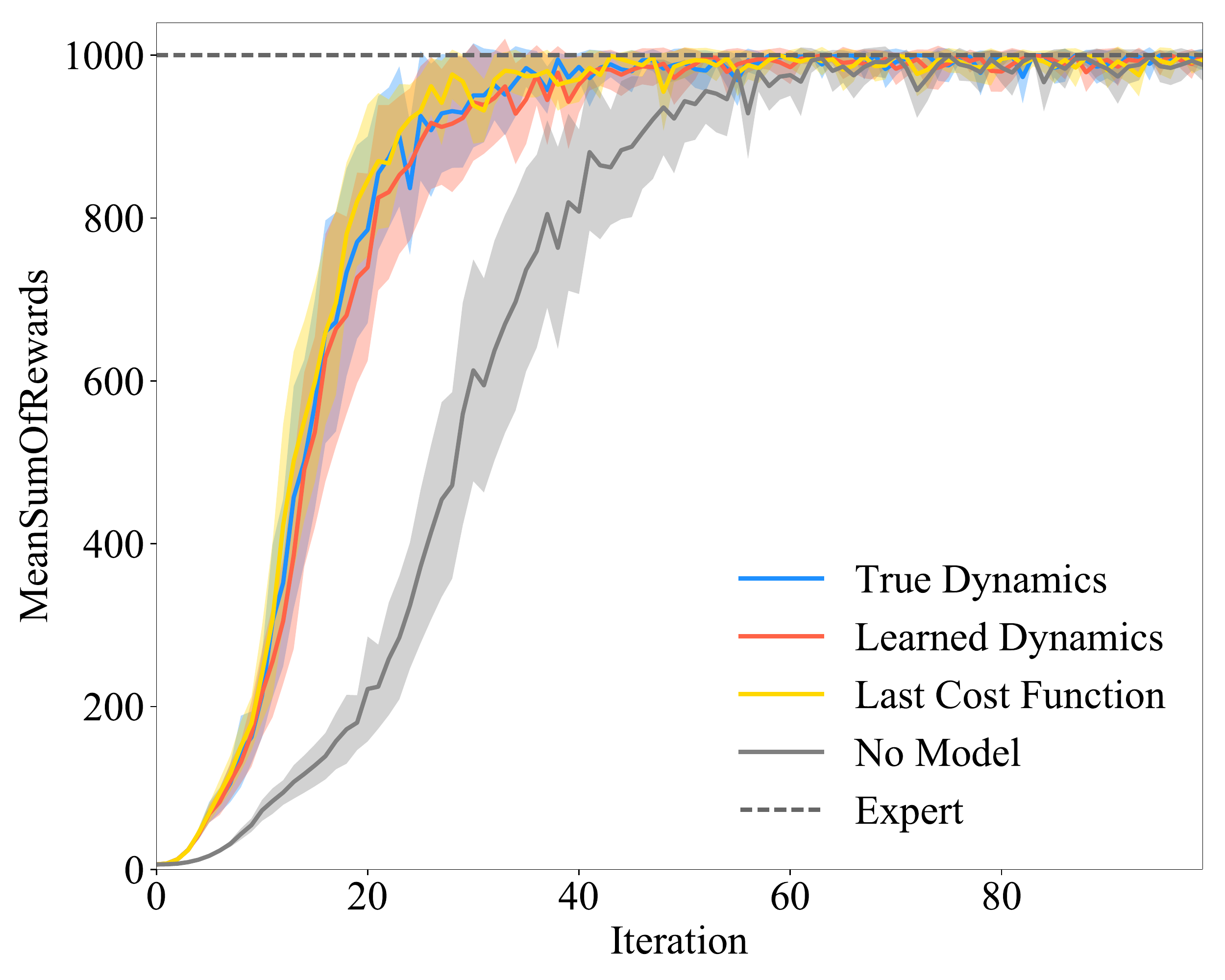}
	\end{subfigure}
	\begin{subfigure}{.24\textwidth}
		\includegraphics[trim={1.0cm 1.15cm 1.2cm 0.0cm}, clip, width=\textwidth]{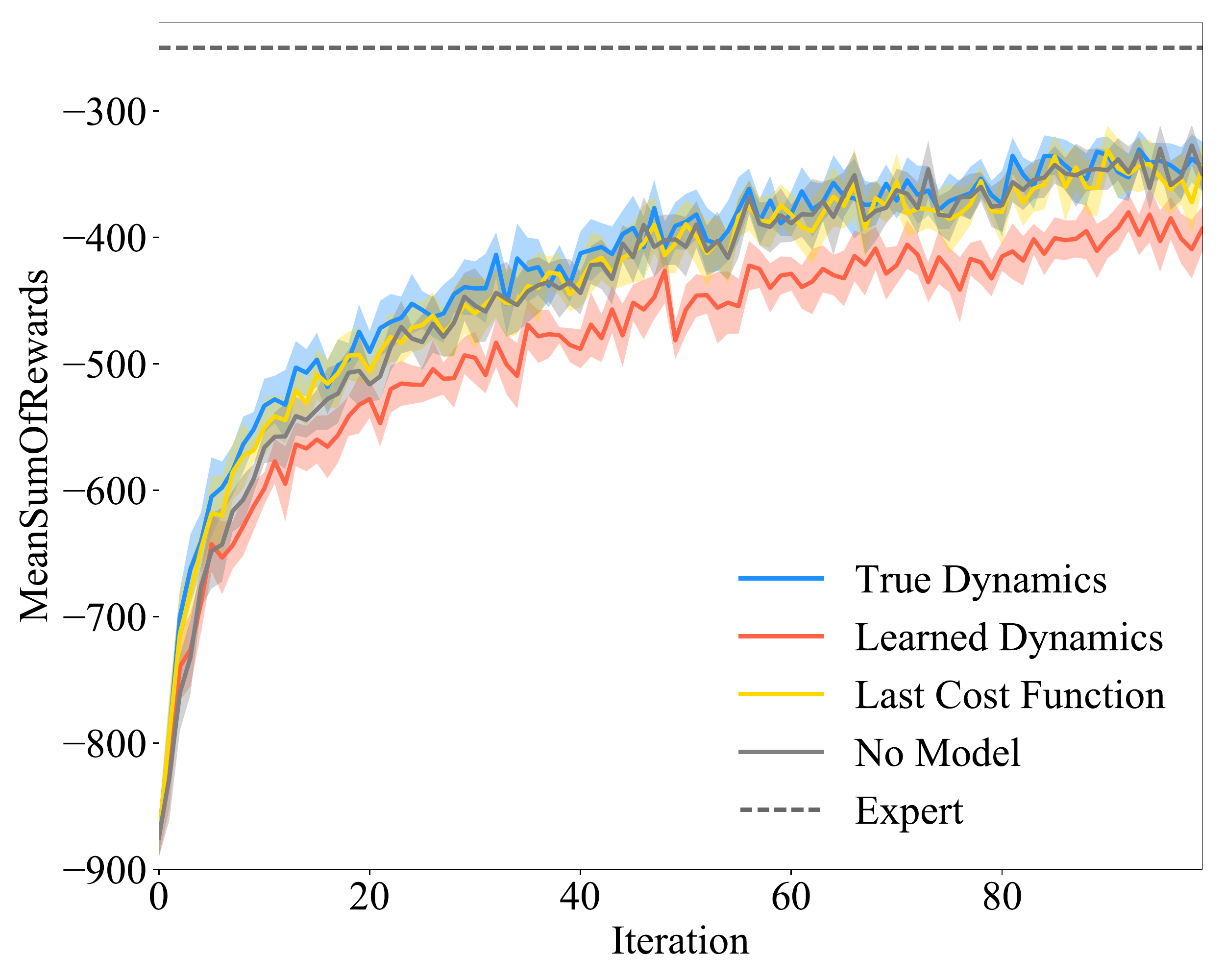}
	\end{subfigure}
	\begin{subfigure}{.24\textwidth}
		\includegraphics[trim={1.0cm 1.15cm 1.2cm 0.0cm}, clip,, width=\textwidth]{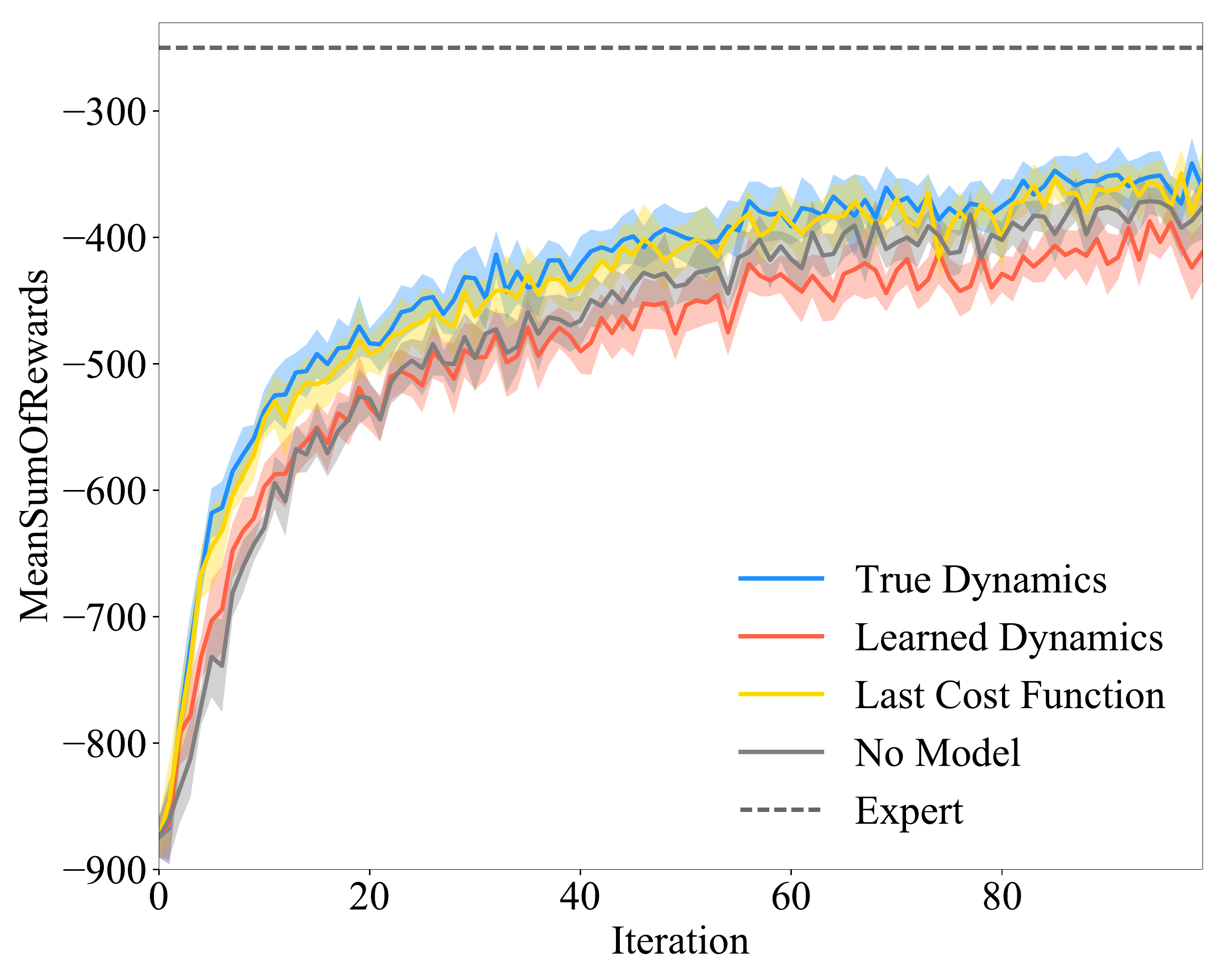}
	\end{subfigure}	\\
	\centering
	\begin{subfigure}{.24\textwidth}
		\includegraphics[trim={1.0cm 1.15cm 1.2cm 0cm}, clip, width=\textwidth]{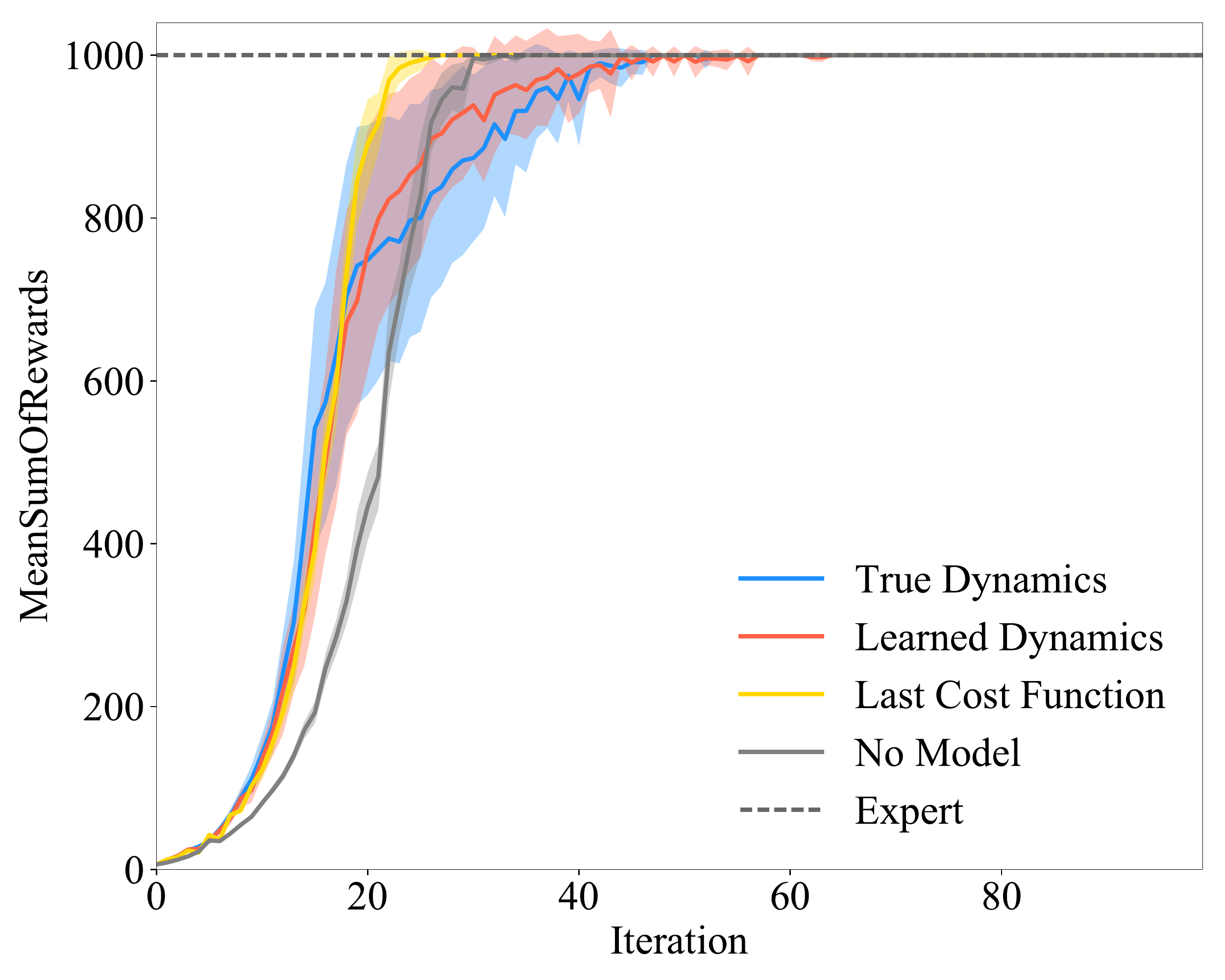}
		\caption{CartPole $p=0$}
	\end{subfigure}
	\begin{subfigure}{.24\textwidth}
		\includegraphics[trim={1.0cm 1.15cm 1.2cm 0cm}, clip, width=\textwidth]{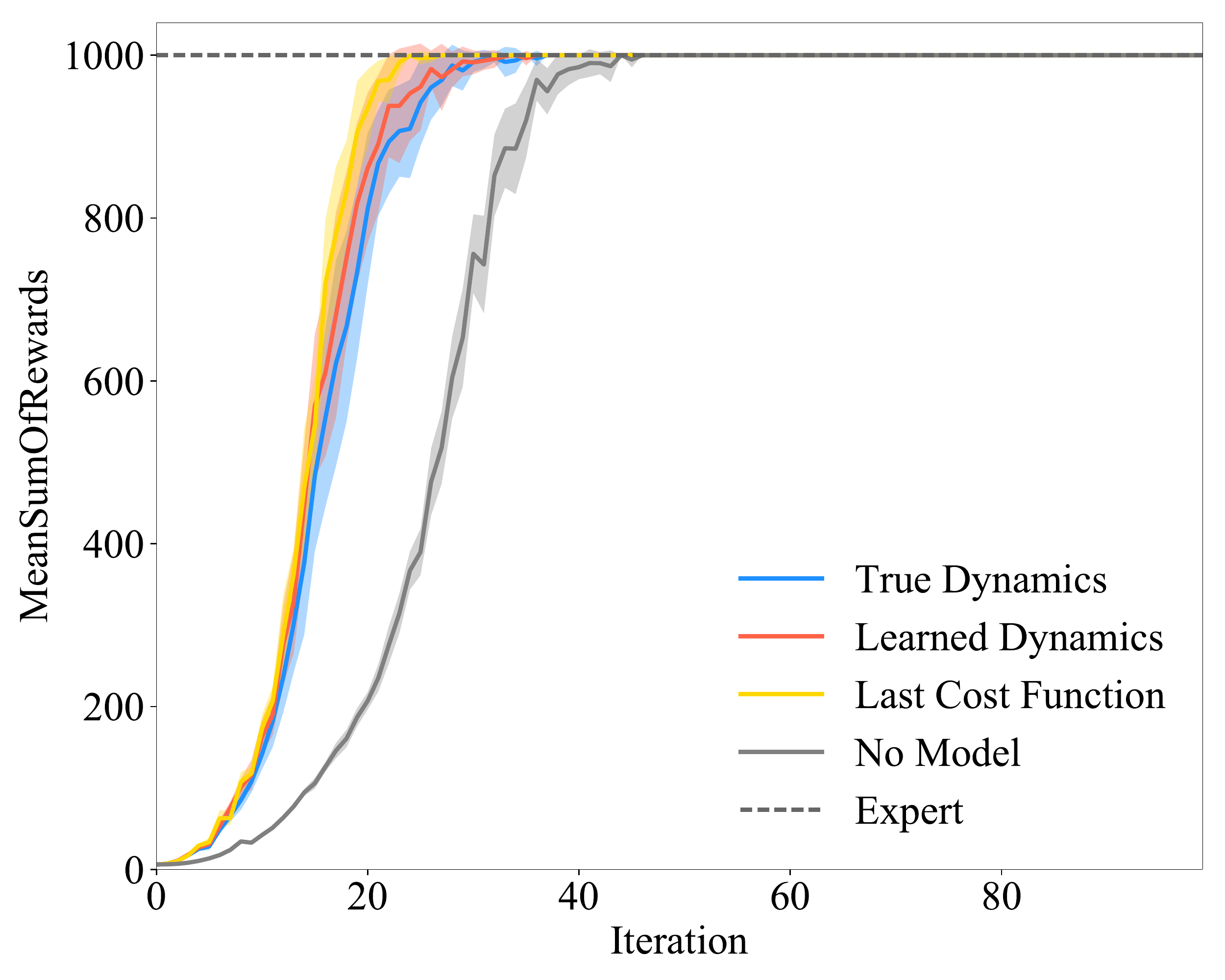}
		\caption{CartPole  $p=2$}
	\end{subfigure}
	\begin{subfigure}{.24\textwidth}
		\includegraphics[trim={1.0cm 1.15cm 1.2cm 0.0cm}, clip, width=\textwidth]{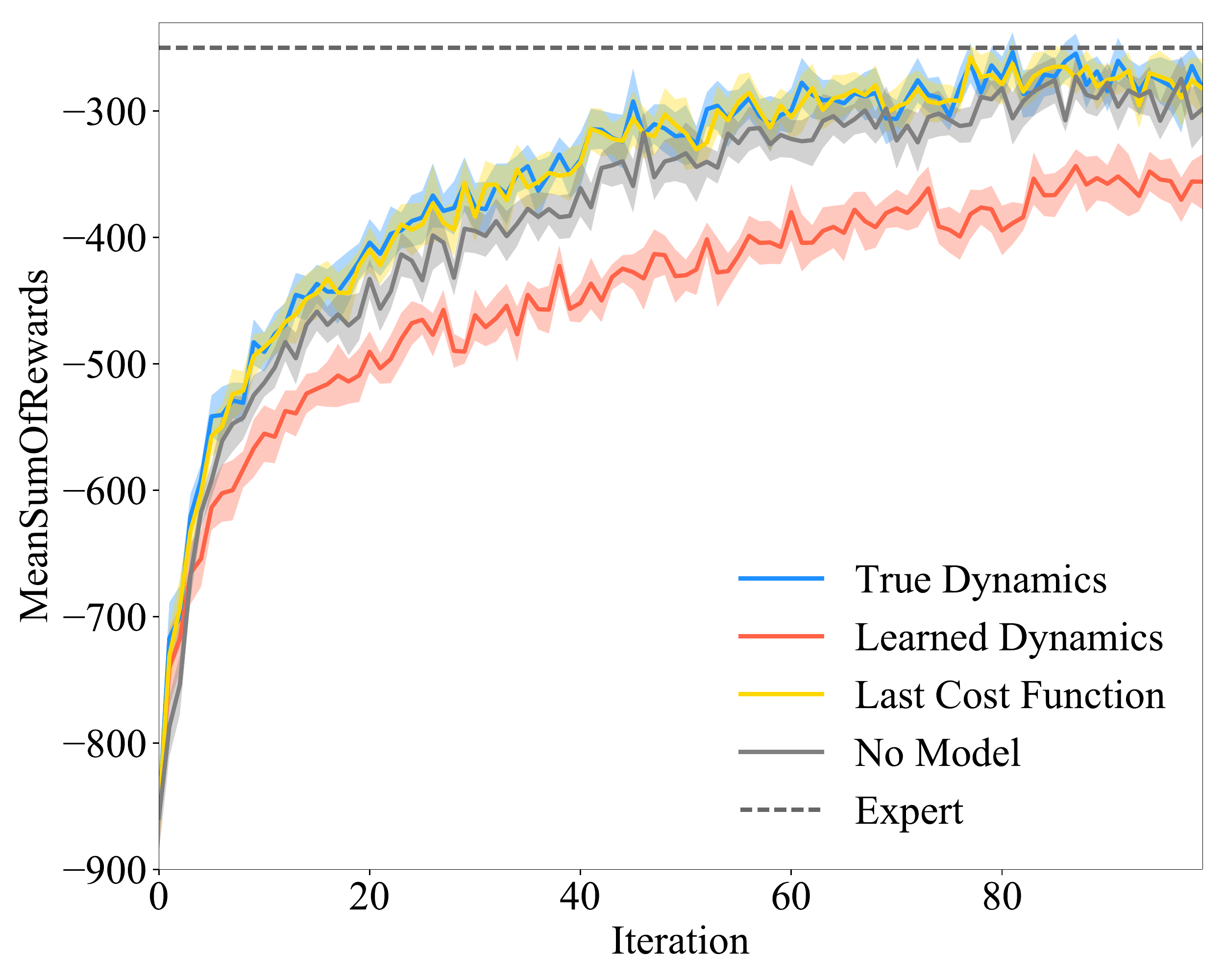}
		\caption{Reacher3D  $p=0$}
	\end{subfigure}
	\begin{subfigure}{.24\textwidth}
		\includegraphics[trim={1.0cm 1.15cm 1.2cm 0.0cm}, clip,, width=\textwidth]{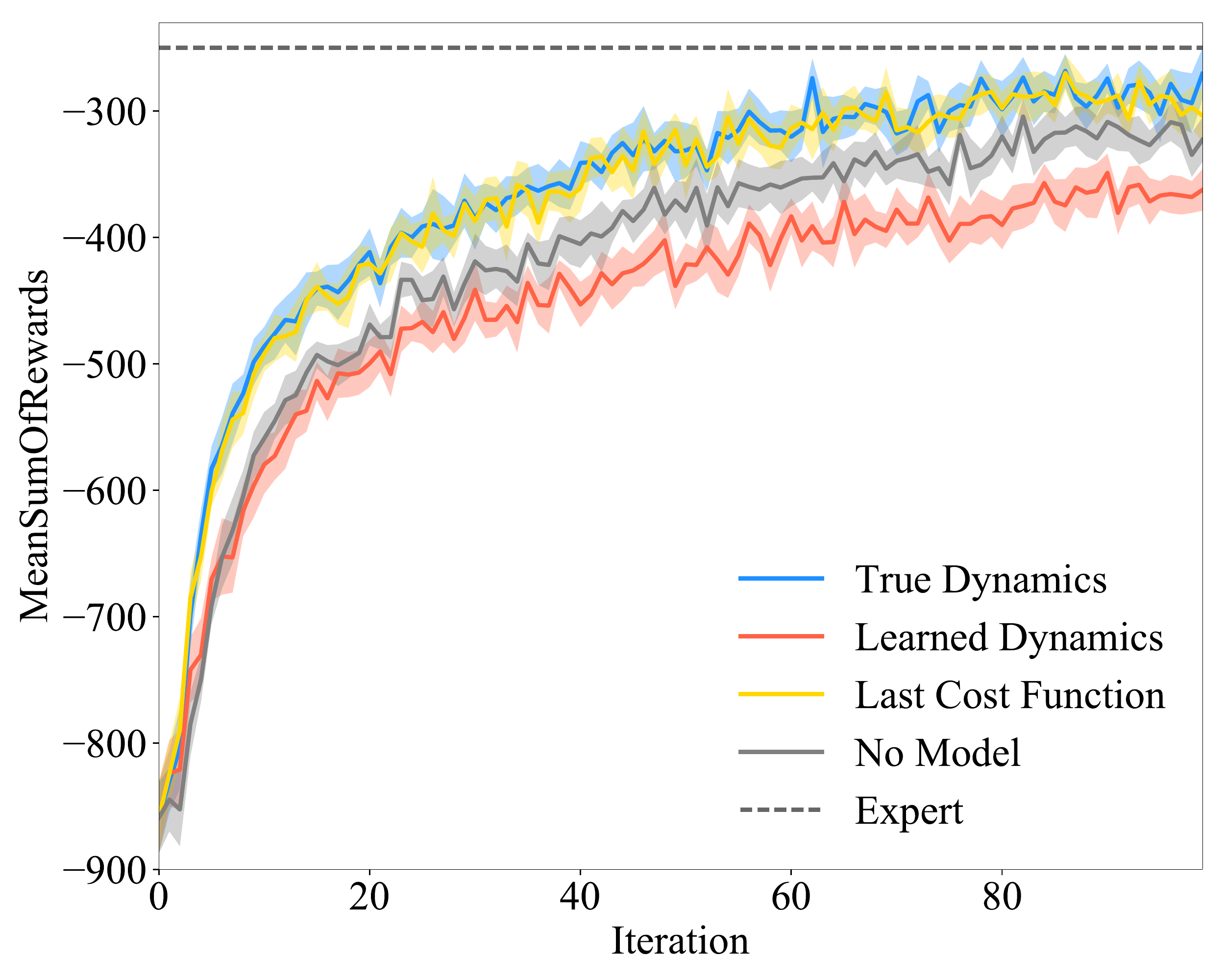}
		\caption{Reacher3D  $p=2$}
	\end{subfigure}	
	\caption{Experimental results of \algprox with neural network (1st row) and linear policies (2nd row). The shaded regions represent $0.5$ standard deviation}
	\label{fig:exps}
\vspace{-5mm}
\end{figure*}

\vspace{-1mm}
\section{EXPERIMENTS} \label{sec:experiment}
\vspace{-2mm}

We experimented with \algprox in simulation to study how weights $w_n = n^p$  and the choice of model oracles affect the learning.  We used two weight schedules: $p = 0$ as baseline, and $p = 2$ suggested by Theorem~\ref{th:weighted performance of practical algorithm}. And we considered several predictive models: (a) a simulator with the true dynamics (b) a simulator with online-learned  dynamics (c) the last cost function (i.e. $\hat{g}_{n+1} = \nabla \tilde{f}_n(\hat{\pi}_{n+1})$ (d) no model (i.e. $\hat{g}_{n+1} =0 $; in this case \algprox reduces to the first-order version of \dagger~\citep{cheng2018fast}, which is considered as a baseline here). 

\vspace{-1mm}
\subsection{Setup and Results}
\vspace{-2mm}
Two robot control tasks (CartPole and Reacher3D) powered by the DART physics engine~\citep{lee2018dart} were used as the task environments. The learner was either a linear policy or a small neural network. For each IL problem, an expert policy that shares the same architecture as the learner was used, which was trained using policy gradients. While sharing the same architecture is not required in IL, here we adopted this constraint to remove the bias due to the mismatch between policy class and the expert policy to clarify the experimental results.
For \algprox, we set $r_n(\pi) = \frac{\mu_f \alpha_n}{2} \norm{\pi -\pi_n}^2$ and set $\alpha_n$ such that $ \sum w_n \alpha_n \mu_f = (1 + cn^{p+1/2}) / \eta_n$, where $c=0.1$ and $\eta_n$ was adaptive to the norm of the prediction error. This leads to an effective learning rate $\eta_n w^p /(1 + c n^{p+1/2})$ which is optimal in the convex setting (cf. Table~\ref{table:convex cases}). For the dynamics model, we used a neural network and trained it using FTL.
The results reported are averaged over 24 (CartPole) and 12 (Reacher3D) seeds.
Figure~\ref{fig:exps} shows the results of \algprox.
While the use of neural network policies violates the convexity assumptions in the analysis, it is interesting to see how \algprox performs in this more practical setting. We include the experiment details in Appendix~\ref{app:experiments} for completeness.

\vspace{-1mm}
\subsection{Discussions}
\vspace{-2mm}

We observe that, when $p=0$, having model information does not improve the performance much over standard online IL (i.e. no model), as suggested in Proposition~\ref{pr:constant regret of conceptual algorithm}. By contrast, when $p=2$ (as suggested by Theorem~\ref{th:weighted performance of practical algorithm}), \algprox improves the convergence and performs better than not using models.\footnote{We note that the curves between $p=0$ and $p=2$ are not directly comparable; we should only compare methods within the same $p$ setting as the optimal step size varies with $p$. The multiplier on the step size was chosen such that \algprox performs similarly in both settings.} It is interesting to see that this trend also applies to neural network policies.

From Figure~\ref{fig:exps}, we can also study how the choice of predictive models affects the convergence. As suggested in Theorem~\ref{th:weighted performance of practical algorithm}, \algprox improves the convergence only when the model makes non-trivial predictions. If the model is very incorrect, then \algprox can be slower. 
This can be seen from the performance of \algprox with online learned dynamics models. In the low-dimensional case of CartPole, the simple neural network predicts the dynamics well, and \algprox with the learned dynamics performs similarly as \algprox with the true dynamics. However, in the high-dimensional Reacher3D problem, the learned dynamics model generalizes less well, creating a performance gap between \algprox using the true dynamics and that using the learned dynamics. 
We note that \algprox would still converge at the end despite the model error.
Finally, we find that the performance of \mobil with the last-cost predictive model is often similar to \algprox with the simulated gradients computed through the true dynamics.

\vspace{-1mm}
\section{CONCLUSION}
\vspace{-2mm}

We propose two novel model-based IL algorithms \algprox and \algVI with strong theoretical properties: they are provably up-to-and-order faster than the state-of-the-art IL algorithms and have unbiased performance even when using imperfect predictive models. 
Although we prove the performance under convexity assumptions, we empirically find that \algprox improves the performance even when using neural networks. In general, \mobil accelerates policy learning when having access to an predictive model that can predict future gradients non-trivially.
While the focus of the current paper is theoretical in nature, the design of \mobil leads to several interesting questions that are important to reliable application of \algprox in practice, such as end-to-end learning of predictive models and designing adaptive regularizations for \algprox.

\vspace{-1mm}
\subsubsection*{Acknowledgements}
\vspace{-2mm}
This work was supported in part by NSF NRI Award 1637758 and NSF CAREER Award 1750483.

\bibliography{ref}

\clearpage
\onecolumn
\input{app}

\end{document}

%% file: shortcuts.tex
\usepackage{amsmath}
\usepackage{amsfonts}
\usepackage{amssymb}
\usepackage{amsthm}
\usepackage{bm}
\usepackage{bbm}
\usepackage{mathtools}
\usepackage{enumitem}
\usepackage{thmtools,thm-restate}
\usepackage{algorithm}
\usepackage{algorithmic}
\usepackage{color}
\usepackage{graphicx}
\usepackage{comment}
\theoremstyle{plain}
\newtheorem{lemma}{Lemma}[section]
\newtheorem{theorem}{Theorem}[section]
\newtheorem{proposition}{Proposition}[section]
\newtheorem{corollary}{Corollary}[section]
\newtheorem{assumption}{Assumption}[section]

\theoremstyle{definition}
\newtheorem{definition}{Definition}[section]

\theoremstyle{remark}

\def\EE{\mathcal{E}}\def\FF{\mathcal{F}}

\def\MM{\mathcal{M}}
\def\RR{\mathcal{R}}

\def\XX{\mathcal{X}}

\def\Abb{\mathbb{A}}
\def\Ebb{\mathbb{E}}

\def\Nbb{\mathbb{N}}
\def\Rbb{\mathbb{R}}
\def\Sbb{\mathbb{S}}

\DeclarePairedDelimiter\ceil{\lceil}{\rceil}

\def\R{\Rbb}\def\N{\Nbb}

\newcommand{\norm}[1]{ \| #1 \|  }

\newcommand{\paren}[1]{ \left( #1 \right) }
\newcommand{\bracket}[1]{ \left[#1 \right] }

\newcommand{\lr}[2]{ \left\langle #1, #2 \right\rangle}
\DeclareMathOperator*{\argmin}{arg\,min}
\DeclareMathOperator*{\argmax}{arg\,max}

\newcommand{\KL}[2]{\textrm{KL}[#1 || #2  ]}

\newcommand{\E}{\Ebb}
\newcommand{\expect}[1]{\E\bracket{#1}}

\newcommand{\expectb}[1]{\E\big[#1\big]}
\newcommand{\expectB}[1]{\E\Big[#1\Big]}

\usepackage{etex,etoolbox}
\usepackage{environ}

\makeatletter
\providecommand{\@fourthoffour}[4]{#4}
\newcommand\fixstatement[2][\proofname\space of]{%
	\ifcsname thmt@original@#2\endcsname
	\AtEndEnvironment{#2}{%
		\xdef\pat@label{\expandafter\expandafter\expandafter
			\@fourthoffour\csname thmt@original@#2\endcsname\space\@currentlabel}%
		\xdef\pat@proofof{\@nameuse{pat@proofof@#2}}%
	}%
	\else
	\AtEndEnvironment{#2}{%
		\xdef\pat@label{\expandafter\expandafter\expandafter
			\@fourthoffour\csname #1\endcsname\space\@currentlabel}%
		\xdef\pat@proofof{\@nameuse{pat@proofof@#2}}%
	}%
	\fi
	\@namedef{pat@proofof@#2}{#1}%
}

\globtoksblk\prooftoks{1000}
\newcounter{proofcount}

\NewEnviron{proofatend}{%
	\edef\next{%
		\noexpand\begin{proof}[\pat@proofof\space\pat@label]%
			\unexpanded\expandafter{\BODY}}%
		\global\toks\numexpr\prooftoks+\value{proofcount}\relax=\expandafter{\next\end{proof}}
	\stepcounter{proofcount}}

\def\printproofs{%
	\count@=\z@
	\loop
	\the\toks\numexpr\prooftoks+\count@\relax
	\ifnum\count@<\value{proofcount}%
	\advance\count@\@ne
	\repeat}
\makeatother

\fixstatement{lemma}
\fixstatement{theorem}
\fixstatement{proposition}
\fixstatement{corollary}

%% file: app.tex
\appendix
\numberwithin{equation}{section}

\section{Notation} \label{app:notation}

\begin{table}[h!]
	\caption{Summary of the symbols used in the paper}
	\label{table:symbols}
	\centering
	\begin{tabular}{cc}		
		\toprule
		Symbol & Definition \\
		\midrule		
		$N$ & the total number of rounds in online learning \\
		$J(\pi)$    & the average accumulated cost, $\E_{d_\pi} \E_\pi [c_t]$ of RL in~\eqref{eq:RL problem}    \\
		$d_\pi $  & the generalized stationary state distribution   \\
		$D(q || p)$ & the difference between distributions $p$ and $q$     \\
		$\pi^*$ & the expert policy \\
		$\Pi$ & the hypothesis class of policies \\		
		$\pi_n$ & the policy run in the environment at the $n$th online learning iteration  \\
		$\hat \FF$ & the hypothesis class of models (elements denoted as $\hat F$)\\		
		$\hat F_n$ & the model used at the $n-1$ iteration to predict the future gradient of the $n$th  iteration \\
		$\epsp$ & the policy class complexity (Definition~\ref{def:general class complexity}) \\
		$\epsf$ & the model class complexity (Definition~\ref{def:general class complexity})\\
		$F(\pi', \pi)$  & the bivariate function $E_{d_{\pi}'} [D(\pi^* || \pi)]$ in~\eqref{eq:bivariate function} \\
		$f_n(\pi)$ & $F(\pi_n, \pi)$ in~\eqref{eq:policy per-round cost} \\
		$\tilde f_n(\pi)$ &  an unbiased estimate of $f_n(\pi)$ \\
		$h_n(\hat F)$ &  an upper bound of $\norm{\nabla_2 F(\pi_n, \pi_n) - \nabla_2 \hat F(\pi_n, \pi_n)}_*^2$\\
		$ \tilde h_n(\hat F)$ &  an unbiased estimate of $h_n(\hat F)$ \\		
		$\mu_f$ & the modulus of strongly convexity of $\tilde f_n$  (Assumption~\ref{as:function structure}) \\
		$G_f$ & an upper bound of $\norm{\nabla \tilde f_n}_* $  (Assumption~\ref{as:function structure})\\
		$G$ & an upper bound of  $\norm{\nabla f_n}_*$  (Theorem~\ref{th:dagger performance})\\
		$\mu_h$ & modulus of strongly convexity of $\tilde h_n$  (Assumption~\ref{as:function structure})\\
		$G_h$ & an  upper bound of  $\norm{\nabla \tilde{h}_n}_*$ (Assumption~\ref{as:function structure})\\		
		$L$ & the Lipschitz constant such that
		$\norm{\nabla_2 \hat F(\pi, \pi) - \nabla_2 \hat F(\pi', \pi')}_* \le L \norm{\pi - \pi'}$ (Assumption~\ref{as:Lipschitz continuity})\\
		$\RR(p) $ & the expected weighted average regret, $\E\left[\frac{\regret^w(\Pi)}{w_{1:N}} \right]$ in~\eqref{eq:RR}\\
		$\regretw$ & the weighted regret, defined in Lemma~\ref{lm:reduction lemma} \\
		$\{w_n\}$ & the sequence of weights used to define $\regretw$; we set $w_n = n^p$ \\
		\bottomrule
	\end{tabular}	
	\vspace{-3mm}
\end{table}

\section{Imitation Learning Objective Function and Choice of Distance} \label{app:choice of discussion}

Here we provide a short introduction to the objective function of IL in~\eqref{eq:IL problem}. The idea of IL is based on the Performance Difference Lemma, whose proof can be found, e.g. in~\citep{kakade2002approximately}.

\begin{lemma}[Performance Difference Lemma] \label{lm:performance difference}
	Let $\pi$ and $\pi'$ be two policies and  $ A_{\pi',t}(s, a) = Q_{\pi',t}(s,a) - V_{\pi',t}(s)$ be the (dis)advantage function  with respect to running $\pi'$. Then it holds that
	\begin{align} \label{eq:performance difference}
	J(\pi) = J(\pi') + \E_{d_\pi} \E_{\pi} [ A_{\pi',t}  ].   
	\end{align}
\end{lemma}

Using Lemma~\ref{lm:performance difference}, we can relate the performance of the learner's policy and the expert policy as 
\begin{align*}
J(\pi) &= J(\pi^*) + \E_{d_\pi} \E_{\pi} [ A_{\pi^*,t}  ]   \\
&= J(\pi^*)  + \E_{d_\pi} \left[ (\E_{\pi} -\E_{\pi^*}) [ Q_{\pi^*,t}  ] \right]
\end{align*}
where the last equality uses the definition of $ A_{\pi',t}$ and that $V_{\pi,t} = \E_{\pi}[Q_{\pi,t}]$.
Therefore, if the inequality below holds
\begin{align*} 
\E_{a\sim \pi_s}[Q_{\pi^*,t}(s,a)]  - \E_{a^* \sim \pi_s^*}[Q_{\pi^*,t}(s,a^*)]    \leq C_{\pi^*} D(\pi_s^*||\pi_s), \qquad \forall t \in \Nbb, s \in \Sbb, \pi \in \Pi
\end{align*}
then minimizing~\eqref{eq:IL problem} would minimize the performance difference between the policies as in~\eqref{eq:IL upper bound}
\begin{align*} 
J(\pi) - J(\pi^*)   \leq  C_{\pi^*} \E_{ d_{\pi}} [D(\pi^*||\pi) ].
\end{align*}

Intuitively, we can set $D(\pi^*||\pi) = \E_{\pi}[ A_{\pi^*,t}]$ and \eqref{eq:IL upper bound} becomes an equality with $C_{\pi^*} = 1$. This corresponds to the objective function used in \aggrevate by~\citet{ross2014reinforcement}. However, this choice requires $ A_{\pi^*,t}$ to be given as a function or to be estimated online, which may be inconvenient or complicated in some settings.

Therefore, $D$ is usually used to construct a strict upper bound in~\eqref{eq:IL upper bound}. The choice of $D$ and $C_{\pi^*}$ is usually derived from  some statistical distances, and it depends on the topology of the action space $\Abb$ and the policy class $\Pi$.
For discrete action spaces, $D$ can be selected as a convex upper bound of the total variational distance between $\pi$ and $\pi^*$ and $C_{\pi^*}$ is a bound on the range of $Q_{\pi^*,t}$ (e.g., a hinge loss used  by~\cite{ross2011reduction}). 
For continuous action spaces, $D$ can be selected as an upper bound of the Wasserstein distance between $\pi$ and $\pi^*$ and $C_{\pi^*}$ is the Lipschitz constant of $Q_{\pi^*,t}$ with respect to action~\citep{pan2017agile}. 
More generally, for stochastic policies, we can simply set $D$ to Kullback-Leibler (KL) divergence (e.g. by~\cite{cheng2018fast}), because it upper bounds both total variational distance and Wasserstein distance. The direction of KL divergence, i.e. $D(\pi_s^*||\pi_s) = \KL{\pi_s}{\pi_s^*}$ or $D(\pi_s^*||\pi_s) = \KL{\pi_s^*}{\pi_s}$, can be chosen based on the characteristics of the expert policy. For example, if the $\log$ probability of the expert policy (e.g. a Gaussian policy) is available, $\KL{\pi_s}{\pi_s^*}$ can be used. If the expert policy is only accessible through stochastic queries, then $\KL{\pi_s^*}{\pi_s}$ is the only feasible option.

\section{Missing Proofs} \label{app:proofs of analysis}

\subsection{Proof of Section~\ref{sec:reduction}} \label{app:proof of reduction}

\reductionLemma*
\begin{proof}[Proof of Lemma~\ref{lm:reduction lemma}]
	By inequality in~\eqref{eq:IL upper bound} and definition of $f_n$, 
	\begin{align*}
	\E\left[ \sum_{n=1}^{N} w_n (J(\pi_n) - J(\pi^*)  )  \right] \leq C_{\pi^*} \E\left[ \sum_{n=1}^{N} w_n f_n(\pi_n)  \right] = C_{\pi^*} \E\left[ \sum_{n=1}^{N} w_n \tilde{f}_n(\pi_n)  \right],
	\end{align*}
	where the last equality is due to $\pi_n$ is non-anticipating. 
	This implies that 
	\begin{align*}
	\E\left[ \sum_{n=1}^{N} w_n J(\pi_n)  \right] &\leq w_{1:N} J(\pi^*) +  C_{\pi^*} \E\left[ \sum_{n=1}^{N} w_n \tilde{f}_n(\pi_n)  \right] \\
	&= w_{1:N} J(\pi^*)  +  C_{\pi^*} \E\left[ \min_{\pi\in\Pi} \sum_{n=1}^{N} w_n \tilde{f}_n(\pi) + \regret^w(\Pi)  \right]
	\end{align*}
	The statement is  obtained by dividing both sides by $w_{1:N}$ and by the definition of $\epsf$.
\end{proof}

\subsection{Proof of Section~\ref{sec:conceptual algorithm (analysis)}} \label{app:proof of conceptual algorithm}

\weightedPerformanceOfConceptualAlgorithm*
\begin{proof}
	We prove a more general version of Theorem~\ref{th:weighted performance of conceptual algorithm} below.
\end{proof}
\begin{theorem}  \label{th:weighted performance of conceptual algorithm (complete)}
	For \algVI, 
	\begin{align*}
	\RR(p) \le 
	\begin{cases}
	\frac{G_h^2}{4\mu_f\mu_h} \frac{p (p+1)^2 e^{\frac{p}{N}}}{p-1} \frac{1}{N^2} + \frac{1}{2\mu_f}\frac{(p+1)^2e^{\frac{p}{N}}} {p} \frac{1}{N}\epsf , 
	&\text{for $p > 1$} \\
	\frac{G_h^2}{\mu_f \mu_h} \frac{\ln (N+1)}{N^2} + \frac{2}{\mu_f}\frac{1}{N}  \epsf,
	&\text{for $p = 1$} \\
	\frac{G_h^2}{ 4  \mu_f \mu_h} (p+1)^2 \frac{O(1)}{N^{p+1}} + \frac{1}{2\mu_f}\frac{(p+1)^2 e^{\frac{p}{N}}}{p} \frac{1}{N^2}\epsf, 
	&\text{for $0<p <1$} \\
	\frac{G_h^2}{2 \mu_f \mu_h}\frac{1}{N} +  \frac{1}{2 \mu_f}\frac{\ln N + 1}{N}\epsf,
	&\text{for $p = 0$} \\
	\end{cases}
	\end{align*}
	
\end{theorem}

\begin{proof}
	The solution $\pi_{n+1}$ of the VI problem~\eqref{eq:VI problem} satisfies the optimality condition of
	\begin{align*}
	\pi_{n+1} = \argmin_{\pi \in \Pi} \sum_{m=1}^n w_m f_m(\pi_n) + w_{n+1} \hat F_{n+1}(\pi_{n+1}, \pi).
	\end{align*}
	Therefore, we can derive the bound of $\RR(p)$\footnote{The expectation of $\RR(p)$ is not required here because \algVI assumes the problem is deterministic.}
	as
	\begin{align}
	\RR(p) &= \frac{\regret^w(\Pi)}{w_{1:N}} \nonumber\\
	&\le \frac{p+1}{2\mu_f w_{1:N} } \sum_{n=1}^{N}  n^{p-1} \norm{\nabla_2 F(\pi_n, \pi_n) - \nabla_2 \hat F_n(\pi_n, \pi_n)}_*^2 & \text{(Lemma~\ref{lm:regret outer})} \nonumber\\
	&\le  \frac{p+1}{2 \mu_f w_{1:N}} \sum_{n=1}^{N}  n^{p-1} h_n(\pi_n)& \text{(Property of $h_n$)}  \label{eq:Rp algvi} 
	\end{align}	
	Next,  we treat $n^{p-1} h_n$ as the per-round cost for an online learning problem, and utilize Lemma~\ref{lm:regret inner} to upper bound the accumulated cost.
	In particular, we set $w_n$ in Lemma~\ref{lm:regret inner}  to $n^{p-1}$ and $l_n$ to $h_n$.
	Finally,  $w_{1:N} = \sumnN n^p$ can be lower bounded using Lemma~\ref{lm:poly sum}.
	Hence, for $p > 1$, we have
	\begin{align*}
	\RR(p) &\le \frac{p+1}{2 \mu_f} \frac{p+1}{N^{p+1}} \paren{
		\frac{G_h^2}{2\mu_h} \frac{p}{p-1} (N+1)^{p-1} + \frac{1}{p} (N+1)^{p}\epsf} \\
	&=	\frac{G_h^2}{4\mu_f\mu_h} \frac{p (p+1)^2 }{p-1} \paren{\frac{N+1}{N}}^{p-1}
	\frac{1}{N^2} + \frac{1}{2\mu_f}\frac{(p+1)^2} {p} 
	\paren{\frac{N+1}{N}}^p
	\frac{1}{N}\epsf \\
	&\le	\frac{G_h^2}{4\mu_f\mu_h} \frac{p (p+1)^2 e^{\frac{p}{N}}}{p-1} \frac{1}{N^2} + \frac{1}{2\mu_f}\frac{(p+1)^2e^{\frac{p}{N}}} {p} \frac{1}{N}\epsf, 
	\end{align*}
	where in the last inequality we utilize the fact that 
	$1 + x  \le e^x, \forall x \in \R$.
	Cases other than $p > 1$ follow  from  straightforward
	algebraic simplification. 
\end{proof}

\constantRegretOfConceptualAlgorithm*
\begin{proof}
	Proved in Theorem~\ref{th:weighted performance of conceptual algorithm (complete)} by setting $p = 0$.
\end{proof}

\subsection{Proof of Section~\ref{sec:practical algorithm (analysis)}} \label{app:proof of practical algorithm}

\strongerFTL*
\begin{proof}
	The proof is based on observing $l_{n} = l_{1:n} - l_{1:n-1}$ and $l_{1:N}$ as a telescoping sum:
	\begin{align*}
	\regret(\XX) &= \sum_{n=1}^N l_n(x_n) - l_{1:N}(x_N^\star) \\
	&= \sum_{n=1}^N \paren{l_{1:n}(x_n) - l_{1:n-1}(x_n)} - 
	\sum_{n=1}^N \paren{  l_{1:n}(x_{n}^\star) -l_{1:n-1}(x_{n-1}^\star)}  \\
	&= \sum_{n=1}^N \paren{ l_{1:n}(x_n) - l_{1:n}(x_n^\star) - \Delta_n},
	\end{align*}
	where for notation simplicity we define $l_{1:0} \equiv 0$.
\end{proof}

\pathwiseRegretBound*

\begin{proof}
	We utilize our new Lemma~\ref{lm:stronger FTL}.
	First, we bound $\sum_{n=1}^N l_{1:n}(\pi_n) - l_{1:n}(\pi_n^\star)$, where
	$\pi_n^\star = \argmin_{\pi \in \Pi} l_{1:n}(\pi)$. We achieve this by Lemma~\ref{lm:regret addon}. Let $l_n = w_n \bar{f}_n = w_n(\lr{g_n}{\pi} + r_n(\pi))$. To use Lemma~\ref{lm:regret addon}, we note that because $r_n$ is centered at $\pi_n$, $\pi_{n+1}$ satisfies
	\begin{align*}
	\pi_{n+1} 
	&= \argmin_{\pi \in \Pi} \sum_{m=1}^{n} w_m \bar{f}(\pi) + w_{n+1} \lr{\hat g_{n+1}}{\pi} \\
	&= \argmin_{\pi \in \Pi} \sum_{m=1}^{n} \underset{l_n(\pi)}{\underbrace{w_m \bar{f}(\pi)  }}
	+
	\underset{v_{n+1}(\pi)}{\underbrace{
			w_{n+1} \lr{\hat g_{n+1}}{\pi} +  w_{n+1} r_{n+1} (\pi_{n+1}) 	
	}}
	\end{align*}
	Because by definition $l_n$ is $w_n \alpha \mu_f$-strongly convex, it follows from Lemma~\ref{lm:regret addon} and Lemma~\ref{lm:poly sum} that 
	\begin{align*}
	\sum_{n=1}^N l_{1:n}(\pi_n) - l_{1:n}(\pi_n^\star)
	&\le \frac{1}{\alpha \mu_f} \sum_{n=1}^N \frac{w_n^2 }{w_{1:n}} \norm{\hat g_n - g_n }_*^2 
	\le \frac{p+1}{2\alpha \mu_f} \sum_{n=1}^N n^{p-1} \norm{ g_{n} - \hat g_{n}}_*^2.
	\end{align*}
	Next, we bound $\Delta_{n+1}$ as follows
	\begin{align*}
	\Delta_{n+1}  
	&=l_{1:n}(\pi_{n+1}) - l_{1:n}(\pi_{n}^\star)\\ 
	&\ge \lr{\nabla l_{1:n}(\pi_{n}^\star)}{\pi_{n+1} - \pi_{n}^\star} + \frac{\alpha\mu_f w_{1:n}}{2} \norm{\pi_{n+1} - \pi_{n}^\star}^2 &\text{(Strong convexity)}\\
	&\ge \frac{\alpha\mu_f w_{1:n}}{2} \norm{\pi_{n+1} - \pi_{n}^\star}^2 & \text{(Optimality condition of $\pi_{n}^*$)} \\
	&= \frac{\alpha \mu_fw_{1:n}}{2} \norm{\pi_{n+1} - \hat \pi_{n+1}}^2  & \text{(Definition of $ \hat \pi_{n+1}$)} \\
	& \ge \frac{\alpha \mu_f n^{p+1}}{2(p+1)} \norm{\pi_{n+1} - \hat {\pi}_{n+1}}^2.& \text{(Definition of $w_n$ and Lemma~\ref{lm:poly sum})} 
	\end{align*}
	Combining these results proves the bound. \qedhere
\end{proof}

\expectedQuadraticError*
\begin{proof}
	By Lemma~\ref{lm:norm squared bound}, we have 
	\begin{align*}
	\expectb{\norm{g_{n} - \hat{g}_{n}}^2_* }\leq 
	4\Big(& \expectb{\norm{g_{n} - \nabla_2 F(\pi_{n}, \pi_{n})}^2_*} +  
	\expectb{\norm{\nabla_2 F(\pi_{n}, \pi_{n}) - \nabla_2 \hat F_{n}(\pi_{n}, \pi_{n})}_*^2} + \\
	& \expectb{\norm{\nabla_2 \hat F_{n}(\pi_{n}, \pi_{n})- \nabla_2 \hat F_{n}(\hat\pi_{n}, \hat \pi_{n})}_*^2} + 
	\expectb{\norm{\nabla_2 \hat F_{n}(\hat \pi_{n}, \hat \pi_{n}) - \hat g_{n}}^2_*}\Big).
	\end{align*}
	Because the random quantities are generated in order
	$\dots, \pi_{n},  g_n, \hat F_{n+1}, \hat {\pi}_{n+1} , \hat g_{n+1}, \pi_{n+1}, g_{n+1}\dots$,
	by the variance assumption (Assumption~\ref{as:variance}), 
	the first and fourth terms can be bounded by
	\begin{align*}
	\expectb{\norm{g_{n} - \nabla_2 F(\pi_{n}, \pi_{n})}^2_*}  &= 
	\E_{\pi_n} \left[ \E_{g_n} [ \norm{g_{n} - \nabla_2 F(\pi_{n}, \pi_{n})}^2_* | \pi_{n} ] \right] \le \sigma_g^2, \\
	\expectb{\norm{\nabla_2 \hat F_{n}(\hat \pi_{n}, \hat\pi_{n}) - \hat g_{n}}^2_*}  &= 
	\E_{\hat{F}_n , \hat{\pi}_n} \left[ \E_{\hat{g}_n}[\norm{\nabla_2 \hat F_{n}(\hat \pi_{n}, \hat \pi_{n}) - \hat g_{n}}^2_* \big| \hat{\pi}_{n}, \hat F_{n} ] \right]  \le \sigma_{\hat{g}}^2.
	\end{align*}
	And, for the second term, we have 
	\begin{align*}
	\expectb{\norm{\nabla_2 F(\pi_{n}, \pi_{n}) - \nabla_2 \hat F_{n}(\pi_{n}, \pi_{n})}_*^2} 	\le
	\expectb{h_n(\hat F_n)} = 	\expectb{\tilde h_n(\hat F_n)} 
	\end{align*}	
	Furthermore, due to the Lipschitz assumption of $\nabla_2 \hat {F}_{n+1}$ (Assumption~\ref{as:Lipschitz continuity}), the third term is bounded by
	\begin{align*}
	\expectb{\norm{\nabla_2 \hat F_{n}(\pi_{n}, \pi_{n})- \nabla_2 \hat F_{n}(\hat \pi_{n}, \hat \pi_{n})}^2_*} \le L^2 \expectb{ \norm{\pi_{n} - {\hat \pi}_{n}}^2 }.
	\end{align*}
	Combing the bounds above, we conclude the lemma.
	
\end{proof}

\weightedPerformanceOfPracticalAlgorithm*
\begin{proof}
	We prove a more general version of Theorem~\ref{th:weighted performance of conceptual algorithm} below.
\end{proof}
\begin{theorem} \label{th:weighted performance of practical algorithm (complete)}
	For \algprox,
	\begin{equation*}
	\begin{split}
	\RR(p) &\le
	\frac{4}{\alpha}\RR_{\mathrm{\algVI}}(p) + \epsp + \sigma(p)\paren{\sigma_g^2 + \sigma_{\hat g}^2}+ 
	\frac{(p+1)\nu_p}{N^{p+1}}, \\
	\sigma(p)& \le 
	\begin{cases}
	\frac{2}{\alpha \mu_f}  \frac{(p+1)^2 e^{\frac{p}{N}} }{p }\frac{1}{N}, & \text{if } p > 0 \\
	\frac{2}{\alpha \mu_f} \frac{\ln N + 1}{N}, & \text{if } p = 0 \\
	\end{cases} \\
	\nu(p) & =  2e\paren{\frac{(p+1)LG_f}{\alpha \mu_f}}^2  \sum_{n=2}^{\nceil} n^{2p-2}
	- 
	\frac{e G_f^2}{2} \sum_{n=2}^{\nceil} (n-1)^{p+1} n^{p-1} = O(1),  \quad \nceil = \ceil{\frac{2e^{\frac{1}{2}}(p+1)LG_f}{\alpha \mu_f}}
	\end{split}
	\end{equation*}
	where $\RR_{\mathrm{\algVI}}(p)$ is the upper bound of the average regret $\RR(p)$  in Theorem~\ref{th:weighted performance of conceptual algorithm (complete)}, and the expectation is due to sampling $\tilde f_n$ and $\tilde h_n$.
\end{theorem}
\begin{proof}
Recall $\RR(p) = \E[\frac{\regretw(\Pi)}{w_{1:N}} ]$, where 
\begin{align*} 
\regret^w(\Pi) = \sum_{n=1}^{N} w_n \tilde{f}_n (\pi_n) - \min_{\pi \in \Pi}  \sum_{n=1}^{N} w_n \tilde{f}_n (\pi).
\end{align*}
Define $\bar f_n(\pi) := \lr{g_n}{\pi} + r_n(\pi)$. Since $\tilde f_n$ is $ \mu_f$-strongly convex, $r_n$ is $ \alpha \mu_f$-strongly convex,  and $r(\pi_n)=0$, $\bar f_n$ satisfies
\begin{align*}
\tilde f_n(\pi_n) - \tilde f_n(\pi) \le \bar f_n(\pi_n) - \bar f_n(\pi), \quad\forall \pi \in \Pi.
\end{align*}
which implies $\RR(p) \leq \E[ \frac{\regretwp(\Pi)}{w_{1:N}}  ]$, where 
\begin{align*}
\regretwp(\Pi) \coloneqq \sum_{n=1}^N w_n \bar f_n(\pi_n) - \min_{\pi \in \Pi} \sum_{n=1}^N w_n \bar f_n(\pi)
\end{align*}
is regret of an online learning problem with per-round cost $ w_n \bar f_n$. 

Lemma~\ref{lm:pathwise regret bound}
upper bounds $\regretwp(\Pi)$ by using Stronger FTL lemma (Lemma~\ref{lm:stronger FTL}).
Since the second term in Lemma~\ref{lm:pathwise regret bound} is negative, which is in our favor, we just need to upper bound the expectation of the first item.  
Using triangular inequality, we proceed to bound $\expectb{\norm{g_{n} - \hat {g}_{n}}_*^2}$, which measures how well we are able to predict the next per-round cost using the model.

By substituting the result of Lemma~\ref{lm:expected quadratic error} into Lemma~\ref{lm:pathwise regret bound}, we see
\begin{equation} \label{eq:regret path loss app}
\begin{split}
\expect{\regretwp(\Pi)} \le \; &
\expectB {\sum_{n=1}^{N} \rho_n \norm{\pi_n - \hat{\pi}_n }^2} + 
\paren{\frac{2(p+1)}{\alpha \mu_f} \sum_{n=1}^N n^{p-1}}\paren{\sigma_g^2 + \sigma_{\hat g}^2} + \\
&\frac{2(p+1)}{\alpha \mu_f}  \expectB{\sum_{n=1}^{N} n^{p-1}\tilde h_n(\hat F_n)} 
\end{split}
\end{equation} 
where 
$\rho_n =\frac{2 (p+1)L^2}{ \alpha \mu_f}n^{p-1} - \frac{\alpha \mu_f}{2(p+1)} (n-1)^{p+1}$.
When $n$ is large enough, $\rho_n \le 0$, and hence the first term of~\eqref{eq:regret path loss app} is $O(1)$.
To be more precise, $\rho_n \le 0$ if 
\begin{align*}
&\frac{2(p+1) L^2}{\alpha \mu_f} n^{p-1}\le  \frac{\alpha \mu_f}{2(p+1)} (n-1)^{p+1} \\
\Longleftrightarrow  \;
&(n-1)^2 \ge \paren{\frac{2(p+1)LG_f}{\alpha \mu_f}}^2
\paren{\frac{n}{n-1}}^{p-1} \\
\Longleftarrow \;
& (n-1)^2 \ge  \paren{\frac{2(p+1)LG_f}{\alpha \mu_f}}^2
e^{\frac{p-1}{n-1}} \\
\Longleftarrow \;
& (n-1)^2 \ge  \paren{\frac{2(p+1)LG_f}{\alpha \mu_f}}^2
e &\text{(Assume $n \ge p$)} \\
\Longleftarrow \;
& n \ge  \frac{2e^{\frac{1}{2}}(p+1)LG_f}{\alpha \mu_f} + 1
\end{align*}
Therefore, we just need to bound the first $\nceil = \ceil{\frac{2e^{\frac{1}{2}}(p+1)LG_f}{\alpha \mu_f}}$ terms of $\rho_n\norm{\pi_n - \hat\pi_n}^2$. Here we use a basic fact of convex analysis in order to bound $\norm{\pi_n - \hat \pi_n}^2$
\begin{lemma} \label{lm:difference between optima}
	Let $\XX$ be a compact and convex set and let $f,g$ be convex functions. Suppose $f+g$ is $\mu$-strongly convex. Let $x_1 \in \argmin_{x \in \XX} f(x)$ and $x_2 = \argmin_{x \in \XX} \paren{f(x)+g(x)}$. Then $\norm{x_1 - x_2} \leq \frac{\norm{\nabla g(x_1)}_*}{\mu}$.
\end{lemma}
\begin{proof} [Proof of Lemma~\ref{lm:difference between optima}]
	Let $h = f+g$. Because $h$ is $\mu$-strongly convex and $x_2 = \argmin_{x\in \XX} h(x)$
	\begin{align*}
	\frac{\mu}{2} \norm{x_1 - x_2}^2 \leq h(x_1) - h(x_2) &\leq \lr{\nabla h(x_1)}{x_1 - x_2} - \frac{\mu}{2} \norm{x_1 - x_2}^2 \\
	&\leq  \lr{ \nabla g(x_1)}{x_1 - x_2} - \frac{\mu}{2} \norm{x_1 - x_2}^2
	\end{align*}
	This implies $\mu \norm{x_1 - x_2}^2 \leq  \lr{ \nabla g(x_1)}{x_1 - x_2}  \leq \norm{\nabla g(x_1)}_* \norm{x_1 - x_2}$. Dividing both sides by $\norm{x_1 - x_2}$ concludes the lemma.
\end{proof}

Utilizing Lemma~\ref{lm:difference between optima} and the definitions of $\pi_n$ and $\hat \pi_n$, we have, for $n \ge 2$, 
\begin{align*}
\norm{\pi_n - \hat \pi_n}^2 &\le \frac{1}{\alpha \mu_f w_{1:n-1}}\norm{w_n\hat g_n}_*^2  \\
&\le \frac{(p+1) G_f^2}{\alpha \mu_f} \frac{n^{2p}}{(n-1)^{p+1}} & \text{(Bounded $\hat g_n$ and Lemma~\ref{lm:poly sum})}\\
&\le \frac{(p+1)e^{\frac{p+1}{n-1}} G_f^2}{\alpha \mu_f}  n^{p-1} &\text{($1+x \le e^x$)} \\
&\le \frac{e (p+1) G_f^2}{\alpha \mu_f}  n^{p-1} &\text{(Assume $n \ge p + 2$)}. 
\end{align*}
and therefore, after assuming initialization $\pi_1  = \hat \pi_1$, we have the bound
\begin{equation}  \label{eq:bound pi}
\sum_{n=2}^{\nceil} \rho_n \norm{\pi_n - \hat \pi_n}^2 
\le
2e\paren{\frac{(p+1)LG_f}{\alpha \mu_f}}^2  \sum_{n=2}^{\nceil} n^{2p-2}
- 
\frac{e G_f^2}{2} \sum_{n=2}^{\nceil} (n-1)^{p+1} n^{p-1} 
\end{equation}
For the third term of~\eqref{eq:regret path loss app}, 
we can tie it back to the bound of $\RR(p)$ of \algVI, which we denote $\RR_{\mathrm{\algVI}}(p)$.
More concretely, recall that for \algVI in~\eqref{eq:Rp algvi}, we have
\begin{align*}
\RR(p)	\le \frac{p+1}{2 \mu_f w_{1:N}} \sum_{n=1}^{N}  n^{p-1} h_n(\pi_n),
\end{align*}
and we derived the upper bound ($\RR_{\mathrm{\algVI}}(p)$) for the RHS term.
By observing that the third term of~\eqref{eq:regret path loss app} after averaging is 
\begin{equation} \label{eq:bound hn}
\begin{split}
\frac{2(p+1)}{\alpha \mu_f w_{1:N}}  \expectB{\sum_{n=1}^{N} n^{p-1}\tilde h_n(\hat F_n)}  
&= \expectB{ \frac{4}{\alpha} \paren{\frac{p+1}{2 \mu_fw_{1:N}} \sum_{n=1}^{N} n^{p-1}\tilde h_n(\hat F_n)}} \\
&\le \frac{4}{\alpha}  \expectB{ \RR_{\mathrm{\algVI}}(p)} \\
&= \frac{4}{\alpha} \RR_{\mathrm{\algVI}}(p).
\end{split}
\end{equation}

Dividing~\eqref{eq:regret path loss app} by $w_{1:N}$, and plugging in~\eqref{eq:bound pi}, \eqref{eq:bound hn}, we see
\begin{align*}
\RR(p) &\leq \E[  \regretwp(\Pi) /w_{1:N} ] \\
&\le \frac{4}{\alpha} \RR_{\mathrm{\algVI}}(p) +  \frac{1}{w_{1:N}} \paren{\nu_p + \paren{\frac{2(p+1)}{\alpha \mu_f} \sum_{n=1}^N n^{p-1}}\paren{\sigma_g^2 + \sigma_{\hat g}^2} }
\end{align*}
where $\nu_p = 2e\paren{\frac{(p+1)LG_f}{\alpha \mu_f}}^2  \sum_{n=2}^{\nceil} n^{2p-2}
- \frac{e G_f^2}{2} \sum_{n=2}^{\nceil} (n-1)^{p+1} n^{p-1}$,
$\nceil = \ceil{\frac{2e^{\frac{1}{2}}(p+1)LG_f}{\alpha \mu_f}}$.

Finally, we consider the case $p > 1$ as stated in Theorem~\ref{th:weighted performance of practical algorithm}
\begin{align*}
\RR(p) &\le \frac{4}{\alpha} \paren{\frac{G_h^2}{4\mu_f\mu_h} \frac{p (p+1)^2 e^{\frac{p}{N}}}{p-1} \frac{1}{N^2} + \frac{1}{2\mu_f}\frac{(p+1)^2e^{\frac{p}{N}}} {p} \frac{1}{N}\epsf} + 
\frac{p+1}{N^{p+1}} \paren{\nu_p + \paren{\frac{2(p+1)}{\alpha \mu_f} \frac{n^p}{p}} \paren{\sigma_g^2 + \sigma_{\hat g}^2} }  \\
&\le \frac{ (p+1)^2 e^{\frac{p}{N}}}{\alpha \mu_f} \paren{\frac{G_h^2}{\mu_h} \frac{p}{p-1} \frac{1}{N^2} + 
	\frac{2}{p} \frac{\sigma_g^2 + \sigma_{\hat g}^2 + \epsf}{N}} + \frac{(p+1) \nu_p}{N^{p+1}},
\end{align*}
where $\nu_p = 2e\paren{\frac{(p+1)LG_f}{\alpha \mu_f}}^2 \paren{\frac{ (\nceil+1)^{2p-1}}{2p-1} -1}
- \frac{e G_f^2}{2}\frac{(\nceil-1)^{2p+1}}{2p+1}$, $\nceil = \ceil{\frac{2e^{\frac{1}{2}}(p+1)LG_f}{\alpha \mu_f}}$.
\end{proof}

\section{Model Learning through Learning Dynamics Models} \label{app:learning dynamics}
So far we have stated model learning rather abstractly, which only requires
$h_n(\hat {F})$ to be an upper bound of  $\norm{\nabla_2 {F} (\pi_{n}, \pi_{n}) - \nabla_2 \hat F(\pi_{n}, \pi_{n})  }_*^2$. 
Now we give a particular example of $h_n$ and $\tilde{h}_n$ when the predictive model is constructed as a simulator with online learned dynamics models.
Specifically, we consider learning a transition model $M \in \MM$ online that induces a bivariate function $\hat F$, where $\MM$ is the class of transition models.
Let $D_{KL}$ denote the KL divergence 
and let $d_{\pi_n}^M$ be the generalized stationary distribution (cf.~\eqref{eq:RL problem}) generated by running policy $\pi_n$ under transition model $M$.
We  define, for $M_n \in \MM$,
$
\hat F_n(\pi', \pi) :=  \E_{d^{M_n}_{\pi'}} [ D(\pi^* ||\pi )]
$. 
We show the error of $\hat F_n$ can be bounded by  the KL-divergence error of $M_n$.
\begin{restatable}{lemma}{errorBoundInMarginalKL} \label{lm:error bound in marginal KL}
	Assume $\nabla D(\pi^* || \cdot)$ is $L_{D}$-Lipschitz continuous with respect to $\norm{\cdot}_*$. It holds that
	$\norm{\nabla_2 F(\pi_{n}, \pi_{n}) - \nabla_2 \hat {F}_n (\pi_{n}, \pi_{n})  }_*^2
	\leq  2^{-1}  (L_D \diam(\Sbb))^2 D_{KL}( d_{\pi_n} || d^{M_n}_{\pi_n}) $.
\end{restatable}
Directly minimizing the marginal KL-divergence $D_{KL}(d_{\pi_n}, {d}^{M_n}_{\pi_n})$ is a nonconvex problem and requires backpropagation through time. 
To make the problem simpler, we further upper bound it in terms of the KL divergence between the true and the modeled \emph{transition probabilities}. 

To make the problem concrete, here we consider $T$-horizon RL problems.
\begin{restatable}{proposition}{errorBoundInExpectedDynamicsKL} \label{pr:error bound in expected dynamics KL}
	For a $T$-horizon problem with dynamics $P$, let $M_n$ be the modeled  dynamics. Then $\exists C>0$ s.t 
	$
	\norm{\nabla_2 {F} (\pi_{n}, \pi_{n}) - \nabla_2 \hat F_n(\pi_{n}, \pi_{n})}_*^2
	\leq  \frac{C}{T} \sum_{t=0}^{T-1} (T-t)\E_{d_{\pi_n,t}}\E_{\pi}\left[ D_{KL}(P || M_{n} ) \right]
	$.
\end{restatable}
Therefore, we can simply takes $h_n$ as the upper bound in Proposition~\ref{pr:error bound in expected dynamics KL}, and $\tilde{h}$ as its empirical approximation by sampling state-action transition triples through running policy $\pi_n$ in the real environment. This construction agrees with the causal relationship assumed in the Section~\ref{sec:conceptual algorithm}.

\subsection{Proofs}
\errorBoundInMarginalKL*
\begin{proof}
	First, we use the definition of dual norm 
	\begin{align} \label{eq:by def of dual norm}
	\norm{\nabla_2 \hat {F} (\pi_{n}, \pi_{n}) - \nabla_2 F(\pi_{n}, \pi_{n})  }_*
	= \max_{\delta: \norm{\delta}\leq 1} (\E_{ d_{\pi_n}} - \E_{ d_{\pi_n}^{M_n}}) \left[ \lr{\delta}{\nabla D(\pi^*|| \pi_n)} \right]  
	\end{align}
	and then we show that $\lr{\delta}{\nabla D(\pi^*|| \pi_n)}$ is $L_D$-Lipschitz continuous: for $\pi, \pi' \in \Pi$,
	\begin{align*}
	\lr{\delta}{\nabla D(\pi^*|| \pi) - \nabla D(\pi^*|| \pi')} \leq \norm{\delta} \norm{ \nabla D(\pi^*|| \pi) - \nabla D(\pi^*|| \pi')}_* \leq L_{D} \norm{\pi - \pi'}
	\end{align*}
	Note in the above equations $\nabla$ is with respect to $D(\pi^*|| \cdot)$. 
	
	Next we bound the right hand side of~\eqref{eq:by def of dual norm} using Wasserstein distance $D_W$, which is defined as follows~\citep{gibbs2002choosing}:  for two probability distributions $p$ and $q$ defined on a metric space 
	$D_W(p, q) \coloneqq \sup_{f:\text{Lip}(f(\cdot)) \leq 1 }  \E_{x\sim p}[f(x)] - \E_{x\sim q} [f(x)] $.
	
	Using the property that $\lr{\delta}{\nabla D(\pi^*|| \pi_n)}$ is $L_D$-Lipschitz continuous, we can derive
	\begin{align*}
	\norm{\nabla_2 \hat {F} (\pi_{n}, \pi_{n}) - \nabla_2 F(\pi_{n}, \pi_{n})  }_*
	\leq L_{D}  D_W( d_{\pi_n}, \hat{d}_{\pi_n})   
	\leq    \frac{L_D \diam(\Sbb)}{\sqrt{2}} \sqrt{D_{KL}( d_{\pi_n} || \hat{d}^n_{\pi_n})}  
	\end{align*}
	in which the last inequality is due to the relationship between $D_{KL}$ and $D_W$~\citep{gibbs2002choosing}. 	
\end{proof}

\errorBoundInExpectedDynamicsKL*
\begin{proof}
	Let $\rho_{\pi,t}$ be the state-action trajectory up to time $t$ generated by running policy $\pi$, and let $\hat{\rho}_{\pi,t}$ be that of the dynamics model. To prove the result, we use a simple fact: 
	\begin{lemma}
		Let $p$ and $q$ be two distributions. 
		\begin{align*}
		\KL{p(x,y)}{q(x,y)} = \KL{p(x)}{q(x)} + \E_{p(x)}\KL{p(y|x)}{q(y|x)}
		\end{align*}
	\end{lemma}
	Then the rest follows from Lemma~\ref{lm:error bound in marginal KL} and the following inequality.
	\begin{align*}
	D_{KL}( d_{\pi_n} || \hat{d}_{\pi_n}) 
	&\leq \frac{1}{T}\sum_{t=0}^{T-1} D_{KL}( \rho_{\pi_n,t} || \hat{\rho}_{\pi_n,t}) \\
	&= \frac{1}{T}\sum_{t=0}^{T-1} \E_{\rho_{\pi_n, t}}\left[ \sum_{\tau=0}^{t-1}  \ln \frac{ p_M(s_{\tau+1}|s_\tau, a_\tau)  }{ p_{\hat{M}}(s_{\tau+1}|s_\tau, a_\tau) } \right] \\
	&=  \frac{1}{T} \sum_{t=0}^{T-1} (T-t)\E_{d_{\pi,t}}\E_{\pi}\left[ D_{KL}(p_M || p_{\hat{M}} ) \right] \qedhere
	\end{align*}
\end{proof}

\section{Relaxation of Strong Convexity Assumption} \label{app:convex cases}
The strong convexity assumption (Assumption~\ref{as:function structure}) can be relaxed to just convexity.
We focus on studying the effect of $\tilde f_n$ and/or $\tilde h_n$ being just convex on $\RR(p)$ in Theorem~\ref{th:dagger performance} and Theorem~\ref{th:weighted performance of practical algorithm} in big-O notation.
Suggested by Lemma~\ref{lm:strong FTL}, when strong convexity is not assumed, additional regularization has to be added in order to keep the stabilization terms $l_{1:n}(x_n) - l_{1:n}(x_n^\star)$ small.

\begin{lemma}[FTRL with prediction]\label{lm:regret outer convex}
	Let $l_n$ be convex with bounded gradient and let $\XX$ be a compact set.
	In round $n$, let regularization $r_n$ be $\mu_n$-strongly convex for some $\mu_n \geq 0$ such that $r_n(x_n) = 0$ and $x_n \in \argmin_{\XX} r_n(x)$,
	and let
	$v_{n+1}$ be a (non)convex function such that  $\sum_{m=1}^n w_m\paren{l_n + r_n} + w_{n+1} v_{n+1}$ is convex.
	Suppose that learner plays Follow-The-Regularized-Leader (FTRL) with prediction, i.e.	
	$x_{n+1} = \argmin_{x\in \XX} \sum_{m=1}^n\left( w_m\paren{l_n + r_n} + w_{n+1} v_{n+1}\right) (x)$,
	and suppose that $\sum_{m=1}^n w_m \mu_m = \Omega(n^k) > 0$  and $\sum_{m=1}^n w_m r_n(x) \leq O(n^k) $ for all $x\in \XX$ and some $k\geq0$.
	Then, for $w_n = n^p$,
	\begin{align*}
	\regretw(\XX) = O(N^k)+ \sum_{n=1}^N O\paren{n^{2p - k}} \norm{\nabla l_n(x_n) - \nabla v_n(x_n)}^2_*
	\end{align*}
	
\end{lemma}
\begin{proof}
	The regret of the online learning problem with \emph{convex} per-round cost $w_n l_n$ can be bounded by the regret of the online learning problem with \emph{strongly convex} per-round cost $w_n \paren{l_n + r_n}$ as follows. 	Let $x_n^\star \in \argmin_{x \in \XX} \sum_{n=1}^N w_n l_n(x)$.
	\begin{align*}
	\regretw(\XX) &= 
	\sum_{n=1}^N w_n l_n(x_n) - \min_{x \in \XX} \sum_{n=1}^N w_n l_n(x) \\
	&= \sum_{n=1}^N w_n \paren{l_n(x_n) + r_n(x_n)} - \sum_{n=1}^N w_n \paren{l_n(x_n^\star)  +r_n(x_n^\star)} + 
	\sum_{n=1}^N w_n r_n(x_n^\star) \\
	& \le  \paren{\sum_{n=1}^N w_n \paren{l_n(x_n) + r_n(x_n)}  - \min_{x \in \XX} \sum_{n=1}^N w_n \paren{l_n(x) + r_n(x)}} + O(N^k).
	\end{align*}
	Since the first  term is the regret of the online learning problem with \emph{strongly convex} per-round cost $w_n\paren{l_n + r_n}$, and $x_{n+1} = \argmin_{\XX} \paren{ \sum_{m=1}^n w_m\paren{l_n + r_n} + w_{n+1} v_{n+1}}$, we can bound the first term via Lemma~\ref{lm:regret outer} by setting $w_n = n^p$ and $\sum_{m=1}^n w_m \mu_m = O(n^k)$.
\end{proof}

The lemma below is a corollary of Lemma~\ref{lm:regret outer convex}.
\begin{lemma}[FTRL] \label{lm:regret inner convex}
	Under the same condition in Lemma~\ref{lm:regret outer convex}, 
	suppose that learner plays FTRL, i.e.
	$x_{n+1} = \argmin_{\XX}\sum_{m=1}^n w_m\paren{l_n + r_n}$. 
	Then, for $w_n = n^p$ with $p > -\frac{1}{2}$, choose $\{ r_n \}$ such that 
	$\sum_{m=1}^n w_m \mu_m = \Omega(n^{p+1/2}) > 0$ and it achieves	
	$
	\regretw(\XX) = O(N^{p+\frac{1}{2}})
	$ and  $\frac{\regretw(\XX)}{w_{1:N}} = O(N^{-1/2})$.
	
\end{lemma}	
\begin{proof}	
	Let $\sum_{m=1}^n w_m \mu_m = \Theta(n^{k})>0$ for some $k\geq 0 $. 
	First, if  $2p - k >  -1$, then we have 
	\begin{align*}
	\regret(\XX) 
	&\le O(N^k) + \sum_{n=1}^N O\paren{n^{2p-k}} \norm{\nabla l_n(x_n)}_*^2 
	&\text{(Lemma~\ref{lm:regret outer convex})} \\
	&\le O(N^k) + \sum_{n=1}^N O\paren{n^{2p-k}} 
	&\text{($l_n$ has bounded gradient)} \\
	&\le O(N^k) + O\paren{N^{2p-k+1}}
	&\text{(Lemma~\ref{lm:poly sum})} \\	
	\end{align*}
	In order to have the best rate, we balance the two terms $O(N^k)$ and $O\paren{N^{2p-k+1}}$
	\begin{align*}
	k = 2p -k + 1 \Longrightarrow k = p + \frac{1}{2},
	\end{align*}
	That is, $p > -\frac{1}{2}$, because $2p - (p+\frac{1}{2}) > -1$.
	This setting achieves regret in $O(N^{p+\frac{1}{2}})$. 
	Because $w_{1:N} = O(N^{p+1})$, the average regret is in $O(N^{-\frac{1}{2}})$.
\end{proof}
With these lemmas, we are ready to derive the upper bounds of $\RR(p)$  when either $\tilde f_n$ or $\tilde h_n$ is just convex, with some minor modification of Algorithm~\ref{alg:algs}. 
For example, when $\tilde{f}_n$ is only convex, $r_n$ will not be $\alpha \mu_f$ strongly; instead we will concern the strongly convexity of $\sum_{m=1}^{n} w_n r_n$. 
Similarly, if $\tilde{h}_n$ is only convex, the model cannot be updated by FTL as in line 5 of Algorithm~\ref{alg:algs}; instead it has to be updated by FTRL.

In the following, we will derive the rate for \algVI (i.e. $\tilde{f}_n = f_n$ and $\tilde{h} = h$) and assume  $\epsilon_{{\hat \FF}}^w = 0 $ for simplicity. 
The same rate applies to the \algprox when there is no noise. To see this, for example, if $\tilde f_n$ is only convex, we can treat $r_n$  as an additional regularization and we can see
\begin{align*}
\RR(p) = \expectB{\frac{\regretw(\Pi)}{w_{1:N}} }
\le  \frac{1}{w_{1:N}} \expectB{	\underbrace{\sum_{n=1}^N w_n \bar f_n(\pi_n) - \min_{\pi \in \Pi} \sum_{n=1}^N w_n \bar f_n(\pi)}_{
		\regretwp(\Pi)}  + \sum_{n=1}^N w_n r_n(\pi^\star_N)}
\end{align*}
where $\pi^\star_N = \argmin_{\pi \in \Pi} \sum_{n=1}^N \tilde f_n(\pi)$. 
As in the proof of Theorem~\ref{th:weighted performance of practical algorithm}, 
$\regretwp$ is decomposed into several terms: the $\tilde h_n$ part in conjunction with $\sum_{n=1}^N w_n r_n(\pi^\star_N)$ constitute the same $\RR(p)$ part for \algVI, while other terms in $\regretwp$ are kept the same.

\paragraph{Strongly convex $\tilde f_n$ and convex $\tilde h_n$}

Here we assume $p > \frac{1}{2}$. Under this condition, we have
\begin{align*}
\regretw(\Pi)  &= \sum_{n=1}^N O(n^{p-1}) \tilde h_n(\hat F_n) 
&\text{(Lemma~\ref{lm:regret outer})} \\
&= O\paren{N^{p-\frac{1}{2}}}
&\text{(Lemma~\ref{lm:regret inner convex})}
\end{align*}
Because $w_{1:N} = \Omega(N^{p+1})$, the average regret $\RR(p) = O(N^{-3/2})$.

\paragraph{Convex $\tilde f_n$ and strongly convex $\tilde h_n$}
Here we assume $p > 0$. 
Suppose $r_{1:n}$ is $\Theta(n^k)$-strongly convex and $2p - k > 0$. 
Under this condition, we have
\begin{align*}
\regretw(\Pi) &= O(N^k)+ \sum_{n=1}^N O\paren{n^{2p - k}} \tilde h_n(\hat F_{n+1}) &\text{(Lemma~\ref{lm:regret outer convex})} \\
&= O\paren{N^k}+ O\paren{N^{2p-k}}. &\text{(Lemma~\ref{lm:regret inner})}
\end{align*}
We balance the two terms and arrive at
\begin{align*}
k = 2p - k \Longrightarrow k = p,
\end{align*}
which satisfies the condition $2p - k > 0$, if $p > 0$.
Because $w_{1:N} = \Omega(N^{p+1})$, the average regret $\RR(p) = O(N^{-1})$.	

\paragraph{Convex $\tilde f_n$ and convex $\tilde h_n$}
Here we assume $p\geq 0$. 
Suppose $r_{1:n}$ is $\Theta(n^k)$-strongly convex and $2p - k > -\frac{1}{2}$.
Under this condition, we have 
\begin{align*}
\regretw(\Pi) &= O(N^k)+ \sum_{n=1}^N O\paren{n^{2p - k}} \tilde h_n(\hat F_{n+1}) &\text{(Lemma~\ref{lm:regret outer convex})} \\
&= O\paren{N^k}+ O\paren{N^{2p -k +\frac{1}{2}}} &\text{(Lemma~\ref{lm:regret outer convex})} 
\end{align*}
We balance the two terms and see
\begin{align*}
k = 2p - k + \frac{1}{2} \Longrightarrow k = p + \frac{1}{4},
\end{align*}	
which satisfies the condition $2p - k > -\frac{1}{2}$, if $p \ge 0$.
Because $w_{1:N} = \Omega(N^{p+1})$, the average regret  $\RR(p) = O(N^{-3/4})$.

\paragraph{Convex $f_n$ without model}
Setting $p = 0$ in Lemma~\ref{lm:regret inner convex}, we have
$\regret(\Pi) = O(N^{\frac{1}{2}})$. 

Therefore, the average regret becomes $O(N^{-\frac{1}{2}})$.

\paragraph{Stochastic problems}
The above rates assume that there is no noise in the gradient and the model is realizable. If the general case, it should be selected $k = p+1$ for strongly convex $\tilde{f}_n$ and $k = p + \frac{1}{2}$ for convex $\tilde{f}_n$. The convergence rate will become $O(\frac{ \epsilon_{{\hat \FF}} + \sigma_g^2 + \sigma_{\hat{g}^2}}{N} )$ and  $O(\frac{ \epsilon_{{\hat \FF}} + \sigma_g^2 + \sigma_{\hat{g}^2}}{\sqrt N} )$, respectively.

\section{Connection with Stochastic Mirror-Prox} \label{app:stochastic mirror-prox}

\def\VI{\mathrm{VI}}
\def\DVI{\mathrm{DVI}}
\def\SOL{\mathrm{SOL}}
\def\DSOL{\mathrm{DSOL}}
\def\Prox{\mathrm{Prox}}
\def\ERR{\mathrm{ERR}}

In this section, we discuss how \algprox generalizes stochastic \mirrorprox by~\citet{juditsky2011solving,nemirovski2004prox} and how the new Stronger FTL Lemma~\ref{lm:stronger FTL} provides more constructive and flexible directions to design new algorithms. 

\subsection{Variational Inequality Problems}

\mirrorprox~\citep{nemirovski2004prox} was first proposed  to solve VI problems with monotone operators, which is a unified framework of ``convex-like'' problems, including convex optimization, convex-concave saddle-point problems, convex multi-player games, and equilibrium problems, etc (see~\citep{facchinei2007finite} for a tutorial). Here we give the definition of VI problems and review some of its basic properties. 

\begin{definition} \label{def:VI problems}
	Let $\XX$ be a convex subset in an Euclidean space $\EE$ and let $F:\XX \to \EE$ be an operator, the \emph{VI problem}, denoted as $\VI(\XX,F)$, is to find a vector $x^* \in \XX$ such that 
	\begin{align*}
	\lr{F( x^*)}{x - x^*} \geq 0, \qquad \forall x \in \XX.
	\end{align*}
	The set of solutions to this problem is denoted as $\SOL(\XX,F)$
\end{definition}

It can be shown that, when $\XX$ is also compact, then $\VI(\XX,F)$ admits at least one solution~\citep{facchinei2007finite}. 
For example, if  $F(x) = \nabla f(x)$ for some  function $f$, then solving $\VI(\XX,F)$ is equivalent to finding stationary points.

VI problems are, in general, more difficult than optimization. To make the problem more structured, we will consider the problems equipped with some \emph{general convex} structure, which we define below. When $F(x) = \nabla f(x)$ for some convex function $f$, the below definitions agree with their convex counterparts. 

\begin{definition} \label{def:monotone properties}
	An operator $F:\XX \to \EE$ is called
	\begin{enumerate}
		\item \emph{pseudo-monotone} on $\XX$ if for all $x, y \in \XX$, 
		\begin{align*}
		\lr{F(y)}{x-y} \geq 0 \implies  \lr{F(x)}{x-y} \geq 0 
		\end{align*}
		\item \emph{monotone} on $\XX$ if for all $x, y \in \XX$, 
		\begin{align*}
		\lr{F(x) - F(y)}{x-y} \geq 0 
		\end{align*}
		\item  \emph{strictly monotone} on $\XX$ if for all $x, y \in \XX$, 
		\begin{align*}
		\lr{F(x) - F(y)}{x-y} > 0 
		\end{align*}
		\item $\mu$-\emph{strongly monotone} on $\XX$ if for all $x, y \in \XX$, 
		\begin{align*}
		\lr{F(x) - F(y)}{x-y} \geq  \mu \norm{x-y}^2
		\end{align*}
	\end{enumerate}
\end{definition}

A VI problem is a special case of general equilibrium problems~\citep{bianchi1996generalized}. Therefore, for a VI problem, we can also define its dual VI problem. 
\begin{definition} \label{def:dual VI problems}
	Given a VI problem $\VI(\XX, F)$, the \emph{dual VI problem}, denoted as $\DVI(\XX,F)$, is to find a vector $x^*_D \in \XX$ such that 
	\begin{align*}
	\lr{F(x)}{x-x^*_D} \geq 0, \qquad \forall x \in \XX.
	\end{align*}
	The set of solutions to this problem is denoted as $\DSOL(\XX,F)$. 
\end{definition}

The solution sets of the primal and the dual VI problems are connected as given in next proposition, whose proof e.g. can be found in~\citep{konnov2000duality}.
\begin{proposition} \label{pr:primal and dual VIs}
	\mbox{}
	\begin{enumerate}
		\item If $F$ is pseudo-monotone, then $\SOL(\XX,F) \subseteq \DSOL(\XX,F)$.
		\item If $F$ is continuous, then $\DSOL(\XX,F) \subseteq \SOL(\XX,F)$.
	\end{enumerate}
\end{proposition}

However, unlike primal VI problems, a dual VI problem does not always have a solution even if $\XX$ is compact. To guarantee the existence of solution to $\DSOL(\XX,F)$ it needs stronger structure, such as pseudo-monotonicity as shown in Proposition~\ref{pr:primal and dual VIs}.
Like solving primal VI problems is related to finding \emph{local} stationary points in optimization, solving dual VI problems is related to finding \emph{global} optima when $F(x) =  \nabla f(x)$ for some function $f$~\citep{komlosi1999stampacchia}.

\subsection{Stochastic Mirror-Prox}

Stochastic \mirrorprox solves a monotone VI problem by indirectly finding a solution to its dual VI problem using stochastic first-order oracles. This is feasible because of Proposition~\ref{pr:primal and dual VIs}. The way it works is as follows: 
given an initial condition $x_1 \in \XX$, it initializes $\hat{x}_1 = x_1$; 
at iteration $n$, it receives unbiased estimates $g_n$ and $\hat{g}_n$ satisfying $\E[g_n] = F(x_n)$ and $\E[\hat{g}_n] = F(\hat{x}_n)$ and then performs updates
\begin{align}  \label{eq:stochastic mirror prox (original)}
\begin{split}
x_{n+1} &= \Prox_{\hat{x}_n }(\gamma_n \hat{g}_n )\\
\hat{x}_{n+1} &= \Prox_{\hat{x}_n }(\gamma_n  g_{n+1} )   
\end{split}
\end{align}
where $\gamma_n > 0$ is the step size, and the proximal operator $\Prox$ is defined as 
\begin{align*}
\Prox_y(g) = \argmin_{x\in\XX} \lr{g}{x} + B_\omega(x||y)
\end{align*}
and $B_\omega(x||y) = \omega(x) - \omega(y) - \lr{\nabla \omega(y)}{x-y}$ is the Bregman divergence with respect to an $\alpha$-strongly convex function $\omega$. At the end, stochastic \mirrorprox outputs
\begin{align*}
\bar{x}_N = \frac{\sum_{n=1}^{N} \gamma_n x_n }{\gamma_{1:n}}
\end{align*}
as the final decision.

For stochastic \mirrorprox, the accuracy of an candidate solution $x$ is based on the error
\begin{align*}
\ERR(x) \coloneqq  \max_{y \in \XX} \lr{F(y)}{x-y}.
\end{align*}
This choice of error follows from the optimality criterion of the dual VI problem in Definition~\ref{def:dual VI problems}. That is, $\ERR(x) \leq 0 $ if and only if $x \in \DSOL(\XX, F)$. From Proposition~\ref{pr:primal and dual VIs}, we know that if the problem is pseudo-monotone, a dual solution is also a primal solution. Furthermore, we can show an approximate dual solution is also an approximate primal solution.

Let $\Omega^2 = \max_{x,y \in \XX} B_\omega (x||y) $.
Now we recap the main theorem of~\citep{juditsky2011solving}.\footnote{Here simplify the condition they made by assuming $F$ is Lipschitz continuous and $g_n$ and $\hat{g}_n$ are unbiased.}
\begin{theorem} \label{th:stochastic mirror prox performance (original)}
	{\normalfont \citep{juditsky2011solving}} 
	Let $F$ be monotone. Assume $F$ is $L$-Lipschitz continuous, i.e.
	\begin{align*}
	\norm{F(x) - F(y)}_* \leq L\norm{x - y} \qquad \forall x,y\in\XX
	\end{align*}
	and for all $n$, the sampled vectors are unbiased and have bounded variance, i.e. 
	\begin{align*}
	\E[g_n] = F(x_n), &\qquad \E[\hat{g}_n] = F(\hat{x}_n)\\
	\E[\norm{g_n - F(x_n)}_*^2 ]  \leq \sigma^2, &\qquad 
	\E[\norm{\hat{g}_n - F(\hat{x}_n)}_*^2 ]  \leq \sigma^2
	\end{align*}
	Then for $\gamma_n = \gamma$ with $0 < \gamma_n \leq \frac{\alpha}{\sqrt{3}L}$, it satisfies that 
	\begin{align*}
	\E[\ERR(\bar{x}_N) ] \leq \frac{2\alpha \Omega^2}{N \gamma} +  \frac{ 7 \gamma \sigma^2 }{\alpha} 
	\end{align*}
	In particular, if 
	$
	\gamma = \min \{ \frac{\alpha}{\sqrt{3} L }, \alpha \Omega\sqrt{\frac{2}{7 N \sigma^2}}  \}
	$, then 
	\begin{align*}
	\E[\ERR(\bar{x}_N) ] \leq \max \left\{ \frac{7}{2}\frac{\Omega^2 L }{\alpha }\frac{1}{N}, \Omega  \sqrt{\frac{14 \sigma^2}{ N }}    \right\}
	\end{align*}
\end{theorem}

If the problem is deterministic, the original bound of \citet{nemirovski2004prox} is as follows. 
\begin{theorem} \label{th:mirror prox performance (original)}
	{\normalfont \citep{nemirovski2004prox}}
	Under the same assumption in Theorem~\ref{th:stochastic mirror prox performance (original)}, suppose the problem is deterministic. For $\gamma \leq \frac{\alpha}{\sqrt{2} L}$,
	\begin{align*}
	\ERR(\bar{x}_N) \leq \sqrt{2}  \frac{ \Omega^2 L}{ \alpha} \frac{1}{N}
	\end{align*}
\end{theorem}
Unlike the uniform scheme above, a recent analysis by~\citet{ho2017exploiting} also provides a performance bound the weighted average version of \mirrorprox when the problem is deterministic.
\begin{theorem} \label{th:mirror prox performance (recent)}
	{\normalfont \citep{ho2017exploiting}}
	Under the same assumption in Theorem~\ref{th:stochastic mirror prox performance (original)},
	suppose the problem is deterministic.
	Let $\{w_n \geq 0\}$ be a sequence of weights and let the step size to be $\gamma_n = \frac{\alpha}{L}\frac{w_{1:n}}{\max_m w_m}$. 
	\begin{align*}
	\ERR(\bar{x}_N) \leq \frac{\Omega^2 L}{\alpha} \frac{\max_n w_n}{w_{1:N}}
	\end{align*}
\end{theorem}
Theorem~\ref{th:mirror prox performance (recent)} (with $w_n = w$) tightens Theorem~\ref{th:stochastic mirror prox performance (original)} and Theorem~\ref{th:mirror prox performance (original)} by a constant factor.

\subsection{Connection with \algprox}

To relate stochastic \mirrorprox and \algprox, we first rename the variables in~\eqref{eq:stochastic mirror prox (original)} by setting $\hat{x}_{n+1} \coloneqq \hat{x}_{n}$  and $\gamma_{n+1} \coloneqq  \gamma_n$
\begin{align*} 
\begin{split}
x_{n+1} &= \Prox_{\hat{x}_n }(\gamma_n \hat{g}_n )\\
\hat{x}_{n+1} &= \Prox_{\hat{x}_n }(\gamma_n  g_{n+1} )   
\end{split}
\qquad \qquad \Longleftrightarrow
\begin{split}
x_{n+1} &= \Prox_{\hat{x}_{n+1} }(\gamma_{n+1} \hat{g}_{n+1} )\\
\hat{x}_{n+2} &= \Prox_{\hat{x}_{n+1} }(\gamma_{n+1}  g_{n+1} )   
\end{split}
\end{align*}
and then reverse the order of updates and write them as
\begin{align} \label{eq:stochastic mirror prox (new index)}
\begin{split}
\hat{x}_{n+1} &= P_{\hat{x}_{n} }(\gamma_{n}  g_{n} ) \\
x_{n+1} &= P_{\hat{x}_{n+1} }(\gamma_{n+1} \hat{g}_{n+1} )
\end{split}
\end{align}

Now we will show that the update in~\eqref{eq:stochastic mirror prox (new index)} is a special case of~\eqref{eq:practical policy update}, which we recall below
\begin{align} \tag{\ref{eq:practical policy update}}
\begin{split}
\hat {\pi}_{n+1} &= \argmin_{\pi \in \Pi} \sum_{m=1}^n w_m\big(\lr{g_m}{\pi} + r_m(\pi) \big),\\
\pi_{n+1} &= \argmin_{\pi \in \Pi} \sum_{m=1}^n w_m\big(\lr{g_m}{\pi} + r_m(\pi) \big) + w_{n+1} \lr{\hat g_{n+1}}{\pi},
\end{split}
\end{align}
That is, we will show that $x_n = \pi_n$ and $\hat{x} = \hat{\pi}_n$ under certain setting.

\begin{proposition} \label{pr:equivalence between algprox and mirrorprox}
	Suppose $w_n = \gamma_n$, $\hat{F}_n = F$, $r_1(\pi) = B_{\omega}(\pi || \pi_1)$ and $r_n = 0$ for $n > 1$. 
	If $\Pi = \XX$ is unconstrained, then $x_n = \pi_n$ and $\hat{x}_n = \hat{\pi}_n$ as defined in~\eqref{eq:stochastic mirror prox (new index)} and~\eqref{eq:practical policy update}.
\end{proposition}
\begin{proof}
	We prove the assertion by induction.
	For $n=1$, it is trivial, since $\pi_1 = \hat{\pi}_1 = x_1  = \hat{x}_1$. Suppose it is true for $n$. We show it also holds for $n+1$.

	We first show $\hat{x}_{n+1}  = \hat{\pi}_{n+1}$. 
	By the optimality condition of $\hat{\pi}_{n+1}$, it holds that 
	\begin{align*}
	0 &= \sum_{m=1}^{n} w_m g_m+ \nabla \omega (\hat{\pi}_{n+1}) - \nabla \omega(\pi_1)  \\
	&= \left( w_n g_{n}+ \nabla \omega (\hat{\pi}_{n+1}) - \nabla \omega (\hat{\pi}_{n})  \right)
	+ \left( \sum_{m=1}^{n-1} w_m g_m  + \nabla \omega (\hat{\pi}_{n})   - \nabla \omega(\pi_1)  \right) \\
	&=  w_n g_{n}+ \nabla \omega (\hat{\pi}_{n+1}) - \nabla \omega (\hat{\pi}_{n})   
	\end{align*}
	where the last equality is by the optimality condition of $\hat{\pi}_n$. 
	This is exactly the optimality condition of $\hat{x}_{n+1}$ given in~\eqref{eq:stochastic mirror prox (new index)}, as $\hat{x}_n = \hat{\pi}_n$ by induction hypothesis and $w_n = \gamma_n$. 
	Finally, because  $\Prox$ is single-valued, it implies $\hat{x}_{n+1}  = \hat{\pi}_{n+1}$.

	Next we show that $\pi_{n+1} = x_{n+1}$. By optimality condition of $\pi_{n+1}$, it holds that
	\begin{align*}
	0 &= w_{n+1} \hat{g}_{n+1} + \sum_{m=1}^{n} w_m g_m+ \nabla \omega ({\pi}_{n+1}) - \nabla \omega(\pi_1)  \\
	&= \left(  w_{n+1} \hat{g}_{n+1}  + \nabla \omega ({\pi}_{n+1}) - \nabla \omega (\hat{\pi}_{n+1})  \right)
	+ \left( \sum_{m=1}^{n} w_m g_m  + \nabla \omega (\hat{\pi}_{n+1})   - \nabla \omega(\pi_1)  \right) \\
	&=   w_{n+1} \hat{g}_{n+1}  + \nabla \omega ({\pi}_{n+1}) - \nabla \omega (\hat{\pi}_{n+1})  
	\end{align*}
	This is the optimality condition also for $x_{n+1}$, since we have shown that $\hat{\pi}_{n+1} = \hat{x}_{n+1}$. The rest of the argument follows similarly as above.
\end{proof}

In other words, stochastic \mirrorprox is a special case of \algprox, when $\hat{F}_n = F$ (i.e. the update of $\pi_n$ also queries the environment not the simulator) and the regularization is constant. 
The condition that $\XX$ and $\Pi$ are unconstrained is necessary to establish the exact equivalence between $\Prox$-based updates and FTL-based updates. This is a known property in the previous studies on the equivalence between lazy mirror descent and FTRL~\citep{mcmahan2017survey}. Therefore, when $\hat{F}_n = F$, we can view \algprox as a lazy version of \mirrorprox. It has been empirical observed the FT(R)L version sometimes empirically perform better than the $\Prox$ version~\citep{mcmahan2017survey}.

With the connection established by Proposition~\ref{pr:equivalence between algprox and mirrorprox}, we can use a minor modification of the strategy used in Theorem~\ref{th:weighted performance of practical algorithm} to prove the performance of \algprox when solving VI problems. 
To show the simplicity of the FTL-style proof compared with the algebraic proof of~\citet{juditsky2011solving}, below we will prove from scratch but only using the new Stronger FTL Lemma (Lemma~\ref{lm:stronger FTL}).

To do so, we introduce a lemma to relate expected regret and $\ERR(\bar{x}_N)$. 
\begin{lemma} \label{lm:regret and VI err}
	Let $F$ be a monotone operator. For any $\{x_n \in \XX \}_{n=1}^N$ and $\{w_n \geq 0 \}$,
	\begin{align*}
	\E[ \ERR(\bar{x}_N) ]\leq \E\left[ \max_{x \in \XX} \frac{1 }{w_{1:N}} \sum_{n=1}^{N} w_n \lr{F(x_n)}{x_n - x}  \right]
	\end{align*}
	where $
	\bar{x}_N = \frac{\sum_{n=1}^{N} w_n x_n }{w_{1:n}}
	$.
\end{lemma}
\begin{proof}
	Let $x^\star \in \argmax_{x \in \XX} \lr{F(x)}{\bar{x}_N-x}$. By monotonicity, for all $x_n$, 
	$
	\lr{F(x^\star)}{x_n -x^\star } \leq \lr{F(x_n)}{x_n - x^\star} 
	$.
	and therefore
	\begin{align*}
	\E[ \ERR(\bar{x}_N) ] &= \E\left[ \frac{1}{w_{1:N}} \sum_{n=1}^{N} w_n \lr{F(x^\star)}{x_n -x^\star } \right]   \\
	&\leq \E\left[ \frac{1 }{w_{1:N}} \sum_{n=1}^{N} w_n \lr{F(x_n)}{x_n - x^\star} \right] 
	\leq  \E\left[  \max_{x \in \XX} \frac{1 }{w_{1:N}} \sum_{n=1}^{N} w_n \lr{F(x_n)}{x_n - x}  \right]
	\end{align*}
\end{proof}

\begin{theorem} \label{th:vi error of algprox}
	Under the same assumption as in Theorem~\ref{th:stochastic mirror prox performance (original)}. 
	Suppose  $w_n = n^p$ and $r_n(x) =\beta_n B_{\omega}(x || x_n)$, where $\beta_n$ is selected such that $\sum_{n=1}^{N} w_n \beta_n = \frac{1}{\eta} n^k$ for some $ k\geq 0$ and $\eta > 0$. 
	If $k > p$, then  
	\begin{align*}
	\E[ \ERR(\bar{x}_N) ] &\leq  \frac{1}{w_{1:N}} \left(\frac{ \alpha  \Omega^2 }{\eta}   N^k + \frac{3 \sigma^2 \eta }{\alpha}\sum_{n=1}^{N} n^{2p-k} \right) + \frac{O(1)}{w_{1:N}} 
	\end{align*}
	
\end{theorem}
\begin{proof}
	To simplify the notation, define  $l_n(x) =w_n (\lr{F(x_n)}{x} + r_n(x))$ and let 
	\begin{align*}
	\regret^w(\XX) &=   \sum_{n=1}^{N} w_n \lr{F(x_n)}{x_n } - \min_{x \in \XX} \sum_{n=1}^{N} w_n \lr{F(x_n)}{x} \\
	\RR^w(\XX) &= \sum_{n=1}^{N} l_n(x_n)  -  \min_{x\in\XX} \sum_{n=1}^{N} l_n(x)
	\end{align*}
	By this definition, it holds that
	\begin{align*}
	\regret^w(\XX) \leq \RR^w(\XX)  + \max_{x\in\XX} \sum_{n=1}^{N} w_n r_n (x)
	\end{align*}

	In the following, we bound the two terms in the upper bound above.
	First, by applying Stronger FTL Lemma (Lemma~\ref{lm:stronger FTL}) with $l_n$ and we can show that 
	\begin{align*}
	\RR^w(\XX) &\leq \sum_{n=1}^{N}  l_{1:n} (x_n)- l_{1:n}( x_{n}^\star) - \Delta_n\\
	&\leq \sum_{n=1}^{N} \frac{\eta}{2 \alpha} n^{2p-k}  \norm{g_n - \hat{g}_n}_*^2  - \frac{\alpha (n-1)^{k-1}}{2 \eta} \norm{x_n - \hat{x}_n}^2
	\end{align*}
	where $x_n^\star \coloneqq \argmax_{x \in \XX} l_{1:n}(x)$. 
	Because by Lemma~\ref{lm:norm squared bound} and Lipschitz continuity of $F$, it holds
	\begin{align}  \label{eq:estimation bound}
	\norm{g_n - \hat{g}_n}_*^2 \leq 3 (L^2 \norm{x_n - \hat{x}_n}^2 + 2\sigma^2 )
	\end{align}
	Therefore, we can bound 
	\begin{align} \label{eq:upper bound of R}
	\RR^w(\XX) \leq  \sum_{n=1}^{N} \left( \frac{3}{2} \frac{L^2 \eta  }{\alpha} n^{2p-k}  - \frac{\alpha}{2 \eta} (n-1)^k \right)  \norm{x_n - \hat{x}_n}^2
	+ \frac{3 \sigma^2 \eta }{\alpha}\sum_{n=1}^{N} n^{2p-k}
	\end{align}
	If $k>p$, then the first term above is $O(1)$ independent of $N$.
	On the other hand, 
	\begin{align} \label{eq:upper bound of radius}
	\max_{x\in\XX} \sum_{n=1}^{N} w_n r_n (x) \leq  \frac{  \alpha  \Omega^2 }{\eta}  N^k
	\end{align}
	Combining the two bounds and Lemma~\ref{lm:regret and VI err}, i.e.
	$
	\E[ \ERR(\bar{x}_N) ] \leq \E\left[ \frac{ \regret^w(\XX)}{w_{1:N}} \right] 
	$
	concludes the proof.
\end{proof}

\paragraph{Deterministic Problems}
For deterministic problems, we specialize the proof  Theorem~\ref{th:vi error of algprox} gives.
We set $k=p=0$, $x_1 = \argmin_{x \in \XX } \omega(x)$, which removes the $2$ factor in \eqref{eq:upper bound of radius}, and modify $3$ to $1$ in~\eqref{eq:estimation bound} (because the problem is deterministic). 
By  recovering the constant in the proof, we can show that 
\begin{align*}
\E[ \ERR(\bar{x}_N) ] &\leq  \frac{1}{N} \left(\frac{  \alpha  \Omega^2 }{\eta}   +  \sum_{n=1}^{N} \left( \frac{1}{2} \frac{L^2 \eta  }{\alpha}   - \frac{\alpha}{2 \eta} \right)  \norm{x_n - \hat{x}_n}^2 \right)  
\end{align*}
Suppose . We choose $\eta$  to make the second term non-positive, i.e. 
\begin{align*}
\frac{1}{2} \frac{L^2 \eta  }{\alpha}  - \frac{\alpha}{2 \eta} \leq 0  \impliedby \eta \leq \frac{\alpha}{L} 
\end{align*}
and the error bound becomes
\begin{align*}
\E[ \ERR(\bar{x}_N) ] &\leq  \frac{  L  \Omega^2}{N}
\end{align*}
This bound and the condition on $\eta$ matches that in~\citep{ho2017exploiting}.

\paragraph{Stochastic Problems}
For stochastic problems, we use the condition specified in Theorem~\ref{th:vi error of algprox}.
Suppose $2p - k > -1$. To balance the second term in~\eqref{eq:upper bound of R} and~\eqref{eq:upper bound of radius}, we choose 
\begin{align*}
2p-k+1 = k \implies k = p + \frac{1}{2}
\end{align*}
To satisfy the hypothesis $2p - k > -1$, it requires $ p > -\frac{1}{2}$. 
Note with this choice, it satisfies the condition $k> p$ required in Theorem~\ref{th:vi error of algprox}.
Therefore, the overall bound becomes
\begin{align*}
\E[ \ERR(\bar{x}_N) ]  &\leq  \frac{1}{w_{1:N}} \left(   \frac{ \alpha  \Omega^2  }{\eta} N^{p+\frac{1}{2}} +  \frac{3 \sigma^2 \eta }{\alpha}\sum_{n=1}^{N} n^{p-\frac{1}{2}} \right) + \frac{O(1)}{w_{1:N}}\\
&\leq \frac{p+1}{N^{p+1}} \left(  \frac{  \alpha  \Omega^2  }{\eta} +  \frac{3\eta  \sigma^2}{\alpha (p+\frac{1}{2})} \right) (N+1)^{p+ \frac{1}{2}}   + \frac{O(1)}{N^{p+1}}  \\
&\leq e^{\frac{p+1/2}{N}} (p+1) \left(  \frac{ \alpha  \Omega^2  }{\eta} +  \frac{3\eta  \sigma^2}{\alpha (p+\frac{1}{2})} \right) N^{-\frac{1}{2}}   + \frac{O(1)}{N^{p+1}}  \\
\end{align*}
where we use Lemma~\ref{lm:poly sum} and  $(\frac{N+1}{N})^{p+1/2} \leq e^{\frac{p+1/2}{N}}$. 
If we set $\eta$ such that 
\begin{align*}
\frac{  \alpha  \Omega^2  }{\eta} =  \frac{3\eta  \sigma^2}{\alpha (p+\frac{1}{2})} \implies \eta = \alpha \frac{\Omega}{\sigma} \sqrt{\frac{p+\frac{1}{2} }{3}}
\end{align*}
Then 
\begin{align} \label{eq:VI error candidate}
\E[ \ERR(\bar{x}_N) ]  &\leq 2 e^{\frac{p+1/2}{N}} (p+1)  \Omega \sigma \sqrt{ \frac{3}{p+\frac{1}{2}}    }  N^{-\frac{1}{2}} + \frac{O(1)}{N^{p+1}} 
\end{align}

For example, if $p=0$, then 
\begin{align*}
\E[ \ERR(\bar{x}_N) ] \leq \frac{O(1)}{N} +   \frac{ 2 \sqrt{6}\sigma\Omega  e^{\frac{p+1/2}{N}} }{\sqrt{N}}
\end{align*}
which matches the bound in by~\citet{juditsky2011solving} with a slightly worse constant. 
We leave a complete study of tuning $p$ as future work.

\subsection{Comparison of stochastic \mirrorprox and \algprox in Imitation Learning}

The major difference between stochastic \mirrorprox and \algprox is whether the gradient from the environment is used to also update the decision~$\pi_{n+1}$. It is used in the \mirrorprox, whereas \algprox uses the estimation from simulation. Therefore, for $N$ iterations, \algprox requires only $N$ interactions, whereas \mirrorprox requires $2N$ interactions.

The price \algprox pays extra when using the estimated gradient is that a secondary online learning problem has to be solved. This shows up in the term, for example of strongly convex problems,
\begin{align*}
\frac{ (p+1) G_h^2}{2 \mu_h}\frac{1}{N^2}  +
\frac{\epsf + \sigma_g^2 + \sigma_{\hat g}^2 }{N}  
\end{align*}
in Theorem~\ref{th:weighted performance of practical algorithm}.
If both gradients are from the environment, then $\epsilon_{{\hat \FF}}^w = 0$ and $\sigma_{{\hat g}}^2 = \sigma_{g}^2$. Therefore, if we ignore the $O(\frac{1}{N^2})$ term, using an estimated gradient to update $\pi_{n+1}$ is preferred, if it requires less interactions to get to the magnitude of  error, i.e.
\begin{align*}
2 \times   2  \sigma_{g}^2 \geq      \epsilon_{{\hat \FF}}^w  + \sigma_{{\hat g}}^2  + \sigma_{g}^2
\end{align*}
in which the multiplier of $2$ on the left-hand side is due to \algprox only requires one interaction per iterations, whereas stochastic \mirrorprox requires two.

Because  $\sigma_g^2$ is usually large in real-world RL problems and $\sigma_{{\hat g}}^2$ can be made close to zero easily (by running more simulations), if our model class is reasonably expressive, then \algprox is preferable. Essentially, this is because \algprox can roughly cut the noise of gradient estimates by half.

The preference over \algprox would be more significant for convex problems, because the error decays slower over iterations (e.g. $\frac{1}{\sqrt{N}}$) and therefore more iterations are required by the stochastic \mirrorprox approach to counter balance the slowness due to using noisy gradient estimator.

\section{Experimental Details} \label{app:experiments}

\subsection{Tasks}
Two robot control tasks (Cartpole and Reacher3D) powered by the DART physics engine [19] were used as the task environments. 

\paragraph{Cartpole}
The Cart-Pole Balancing task is a classic control problem, of which the goal is to keep the pole balanced in an upright posture with force only applied to the cart. 
The state and action spaces are both continuous, with dimension $4$ and $1$, respectively. 
The state includes the horizontal position and velocity of the cart, and the angle and angular velocity of the pole.
The time-horizon of this task is $1000$ steps. There is a small uniformly random perturbation injected to initial state, and the transition is deterministic. The agent receives $+1$ reward for every time step it stays in a predefined region, and a rollout terminates when the agent steps outside the region.

\paragraph{Reacher3D}
In this task, a 5-DOF (degrees-of-freedom) manipulator is controlled to reach a random target position in a 3D space. The reward is the sum of the negative distance to the target point from the finger tip and a control magnitude penalty. The actions correspond to the torques applied to the 5 joints.
The time-horizon of this task is $500$ steps. 
At the beginning of each rollout, 
the target point to reach is reset to a random location.

\subsection{Algorithms}
\paragraph{Policies}
We employed Gaussian policies in our experiments, i.e. for any state $s \in \Sbb$, $\pi_s$ is Gaussian distributed. 
The mean of $\pi_s$ was modeled by either a linear function or a neural network that has $2$ hidden layers of size $32$ and $\tanh$ activation functions. 
The covariance matrix of $\pi_s$ was restricted to be diagonal and independent of state.
The expert policies in the IL experiments
share the same architecture as the corresponding learners (e.g. a linear learner is paired with a linear expert) and were trained using actor-critic-based policy gradients.

\paragraph{Imitation learning loss}
With regard to the IL loss, we set $D(\pi^*_s||\pi_s)$ in~\eqref{eq:IL problem} to be the KL-divergence between the two Gaussian distributions: 
$D(\pi^*_s||\pi_s) = \KL{\pi_s}{\pi^*_s}$.
(We observed that using $\KL{\pi_s}{\pi^*_s}$ converges noticeably faster than using $\KL{\pi^*_s}{\pi_s}$).

\paragraph{Implementation details of \algprox}
The regularization of \algprox was set to  $r_n(\pi) = \frac{\mu_f \alpha_n}{2} \norm{\pi -\pi_n}^2$ such that $ \sum w_n \alpha_n \mu_f = (1 + cn^{p+1/2}) / \eta_n$, where $c=0.1$ and $\eta_n$ was adaptive to the norm of the prediction error. 
Specifically, we used $\eta_n = \eta \lambda_n$: $\eta>0$ and $\lambda_n$ is a moving-average estimator of the norm of $e_n = g_n - \hat{g}_n$ defined as
\begin{align*}
\bar{\lambda}_n &= \beta \bar{\lambda}_{n-1} + (1-\beta) \norm{e_n}_2\\
\lambda_n &= \bar{\lambda}_n/(1-\beta^n)
\end{align*}
where $\beta$ was chosen to be $0.999$. This parameterization is motivated by the form of the optimal step size of \algprox in Theorem~\ref{th:weighted performance of practical algorithm}, and by the need of having adaptive step sizes so different algorithms are more comparable. 
The model-free setting was implemented by setting $\hat{g}_n = 0 $ in \algprox, and the same adaptation rule above was used (which in this case effectively adjusts the learning rate based on $\norm{g_n}$). 
In the experiments, $\eta$ was selected to be $0.1$ and $0.01$ for $p = 0$ and $p = 2$, respectively,
so the areas under the effective learning rate $\eta_n w^p /(1 + c n^{p+1/2})$ for $p=0$ and $p=2$ are close, making \algprox perform similarly in these two settings.

In addition to the update rule of \algprox, a running normalizer, which estimates the upper and the lower bounds of the state space, was used to center the state before it was fed to the policies.

\paragraph{Dynamics model learning}
The dynamics model used in the experiments is deterministic (the true model is deterministic too). It is represented by a neural network with $2$ hidden layers of size $64$ and $\tanh$ activation functions.
Given a batch of transition triples $\{(s_{t_k}, a_{t_k}, s_{t_k+1})\}_{k=1}^K $ collected by running $\pi_n$ under the true dynamics in each round,
we set the per-round cost for model learning as
$\frac{1}{K} \sum_{k=1}^K \norm{s_{t_{k+1}} - M(s_{t_k}, a_{t_k})}_2^2$, where $M$ is the neural network dynamics model.
It can be shown that this loss is an upper bound of  $\norm{\nabla_2 {F} (\pi_{n}, \pi_{n}) - \nabla_2 \hat F_n(\pi_{n}, \pi_{n})}_*^2$
by applying a similar proof as in Appendix~\ref{app:learning dynamics}.
The minimization was achieved through gradient descent using ADAM~\cite{kingma2014adam} with a fixed number of iterations ($2048$) and fixed-sized mini-batches ($128$). The step size of ADAM was set to $0.001$.

\newpage
\section{Useful Lemmas} \label{app:lemmas}

This section summarizes some useful properties of polynomial partial sum, sequence in Banach space, and  variants of FTL in online learning. These results will be useful to the proofs in Appendix~\ref{app:proofs of analysis}.

\subsection{Polynomial Partial Sum} \label{app:polynomial series}

\begin{lemma} \label{lm:poly sum}
	This lemma  provides  estimates of $\sum_{n=1}^N n^p$.
	\begin{enumerate}
		\item For $p > 0$,
		$
		\frac{N^{p+1}}{p+1} = \int_{0}^{N} x^p dx \leq \sum_{n=1}^{N} n^p \leq \int_{1}^{N+1} x^p dx \leq \frac{(N+1)^{p+1}}{p+1}.
		$
		\item For $p = 0$, $\sum_{n=1}^N n^p = N$.
		\item For  $-1< p < 0$, 
		\begin{align*}
		\textstyle
		\frac{(N+1)^{p+1}-1}{p+1} = \int_1^{N+1} x^p dx \le \sum_{n=1}^{N} n^p \le 1 + \int_1^N x^p dx = \frac{N^{p+1} + p}{p+1} \le \frac{(N+1)^{p+1}}{p+1}.
		\end{align*}
		\item For $p = -1$, $\ln (N+1) \le  \sum_{n=1}^N n^p\le \ln N + 1$.
		\item For $p < -1$, $\sum_{n=1}^{N} n^p  \le \frac{N^{p+1} + p}{p+1}  = O(1)$. 
		For $p = -2$,  $\sum_{n=1}^{N} n^p \le \frac{N^{-1} - 2}{-2+1} \le 2$.
	\end{enumerate}
	
\end{lemma}

\begin{lemma} \label{lm:poly special}
	For $p \ge -1$, $N \in \N$,
	\begin{align*}
	S(p) = \sum_{n=1}^N \frac{n^{2p}}{\sum_{m=1}^n m^p}  \le
	\begin{cases} 
	\frac{p+1}{p} (N+1)^p, & \text{for } p > 0 \\
	\ln(N+1), & \text{for } p =0 \\
	O(1), & \text{for } -1 <p <0 \\
	2,  & \text{for } p = -1 \\ 
	\end{cases}.
	\end{align*}
\end{lemma}
\begin{proof}
	If $p \ge 0$, 
	by Lemma~\ref{lm:poly sum}, 	
	\begin{align*}
	S(p) = (p+1) \sum_{n=1}^N  n^{p-1} \le
	\begin{cases}
	\frac{p+1}{p} (N+1)^p, & \text{for } p > 0 \\
	\ln(N+1), & \text{for } p =0 \\
	\end{cases}.
	\end{align*}
	If $ -1 < p < 0$,
	by Lemma~\ref{lm:poly sum}, 
	$
	S(p) \le (p+1) \sum_{n=1}^N \frac{n^{2p}}{(n+1)^{p+1}-1}
	$.
	Let $a_n = \frac{n^{2p}}{(n+1)^{p+1}-1} $,  and $b_n = n^{p-1}$.
	Since $\lim_{n \to \infty} \frac{a_n}{b_n} = 1$ and  by Lemma~\ref{lm:poly sum} $\sum_{n=0}^\infty b_n$ converges,  thus $\sum_{n=0}^\infty a_n$ converges too.
	Finally, if $p = -1$, by Lemma~\ref{lm:poly sum},
	$
	S(-1)\le
	\sum_{n=1}^N \frac{1}{n^2 \ln(n+1)} \le
	\sum_{n=1}^N \frac{1}{n^2} \le 2.
	$
\end{proof}

\subsection{Sequence in Banach Space}
\begin{lemma} \label{lm:norm squared bound}
	Let $\{a = x_0, x_1, \cdots, x_N=b\}$ be a sequence in a Banach space with norm $\norm{\cdot}$. Then for any $N \in \N_+$,	
	$
	\norm{a - b}^2 \le N \sum_{n=1}^{N} \norm{x_{n-1} - x_{n}}^2
	$.
\end{lemma}
\begin{proof}
	First we note that by triangular inequality it satisfies that 
	$
	\norm{a - b} \leq \sum_{n=1}^{N} \norm{x_{n-1} - x_{n}}
	$. Then we use the basic fact that $ 2 a b \leq a^2 + b^2$ in the second inequality below and prove the result.
	\begin{align*}
	\norm{a - b}^2 &\leq \sum_{n=1}^{N} \norm{x_{n-1} - x_{n}}^2 + \sum_{n=1}^N \sum_{m=1;m\neq n}^N  \norm{x_{n-1} - x_{n}}  \norm{x_{m-1} - x_{m}} \\
	&\leq  \sum_{n=1}^{N} \norm{x_{n-1} - x_{n}}^2 + \sum_{n=1}^N \sum_{m=1;m\neq n}^N  \frac{1}{2}\left( \norm{x_{n-1} - x_{n}}^2 +   \norm{x_{m-1} - x_{m}}^2  \right) \\
	&=  \sum_{n=1}^{N} \norm{x_{n-1} - x_{n}}^2 +  \frac{N-1}{2} \sum_{n=1}^N   \norm{x_{n-1} - x_{n}}^2 + \frac{1}{2}\sum_{n=1}^N \sum_{m=1;m\neq n}^N   \norm{x_{m-1} - x_{m}}^2  \\
	&=  \sum_{n=1}^{N} \norm{x_{n-1} - x_{n}}^2 +  (N-1) \sum_{n=1}^N   \norm{x_{n-1} - x_{n}}^2 \\
	&=  N \sum_{n=1}^{N} \norm{x_{n-1} - x_{n}}^2  \qedhere
	\end{align*}
\end{proof}

\subsection{Basic Regret Bounds of Online Learning}

For the paper to be self-contained, we summarize some fundamental results of regret bound when the learner in an online problem updates the decisions by variants of FTL. Here we consider a general setup and therefore use a slightly different notation from the one used in the main paper for policy optimization.

\paragraph{Online Learning Setup} Consider an online convex optimization problem. Let $\XX$ be a compact decision set in a normed space with norm $\norm{\cdot}$. In round $n$, the learner plays $x_n \in \XX$
receives a convex loss $l_n: \XX \to \R$ satisfying  $\norm{\nabla l_n(x_n)}_* \leq G$, and then make a new decision $x_{n+1} \in \XX$.
The \emph{regret} is defined as 
\begin{align*}
\regret(\XX)  = \sum_{n=1}^{N} l_n(x_n) - \min_{x \in \XX} \sum_{n=1}^N l_n(x)
\end{align*}
More generally, let $\{ w_n \in \R_+ \}_{n=1}^N$ be a sequence of weights. The \emph{weighted regret} is defined as
\begin{align*}
\regret^w (\XX)  = \sum_{n=1}^{N} w_n l_n(x_n) - \min_{x \in \XX} \sum_{n=1}^N w_n l_n(x)
\end{align*}
In addition, we define a constant $\epsx$  (which can depend on $\{l_n\}_{n=1}^N$) such that 
\begin{align*}
\epsx \geq \min_{x \in \XX} \frac{\sum_{n=1}^N w_n l_n (x)}{w_{1:N}}.
\end{align*}

In the following, we prove some basic properties of FTL with prediction. At the end, we show the result of FTL as a special case. These results are based on the Strong FTL Lemma (Lemma~\ref{lm:strong FTL}), which can also be proven by Stronger FTL Lemma (Lemma~\ref{lm:stronger FTL}).

\strongFTL*

To use Lemma~\ref{lm:strong FTL}, we first show an intermediate bound.
\begin{lemma} \label{lm:regret addon}
	In round $n$, let $l_{1:n}$ be $\mu_{1:n}$-strongly convex for some $\mu_{1:n} > 0$, and let $v_{n+1}$ be a (non)convex function such that $l_{1:n} + v_{n+1}$ is convex. Suppose the learner plays FTL with prediction, i.e.
	$
	x_{n+1} \in \argmin_{x\in \XX} \left(l_{1:n} + v_{n+1}\right)(x)
	$. 	Then it holds
	\begin{align*}
	\sum_{n=1}^N \paren{l_{1:n}(x_n) - l_{1:n}(x_n^\star)}
	\le \sum_{n=1}^{N} \frac{1}{2 \mu_{1:n}}  \|\nabla l_n(x_{n}) - \nabla v_n(x_{n}) \|_*^2
	\end{align*}
	where $x_n^\star = \argmin_{\XX} \sum_{n=1}^N l_n(x) $.
\end{lemma}
\begin{proof}
	For any $x \in \XX$, since $l_{1:n}$ is $\mu_{1:n}$ strongly convex, we have
	\begin{align} \label{eq:regret addon strong convexity}
	l_{1:n}(x_n) - l_{1:n}(x) 
	&\le  \lr{\nabla l_{1:n}(x_n)}{x_n - x} - \frac{\mu_{1:n}}{2} \norm{x_n- x}^2.
	\end{align}	
	And by the hypothesis $x_{n}= \argmin_{x\in \XX} \left(l_{1:n-1} + v_{n}\right)(x)$, it holds that 
	\begin{align} \label{eq:regret addon min}
	\lr{-\nabla l_{1:n-1} (x_n)- \nabla v_n(x_n)}{x_n - x} \geq 0.
	\end{align}
	Adding~\eqref{eq:regret addon strong convexity} and~\eqref{eq:regret addon min} yields
	\begin{align*}
	l_{1:n}(x_n) - l_{1:n}(x) 
	&\le \lr{\nabla l_n(x_n) - \nabla v_n(x_n)}{x_n-x} - \frac{\mu_{1:n}}{2} \norm{x_n- x}^2 \\ 
	&\le \max_d \lr{\nabla l_n(x_n) - \nabla v_n(x_n)}{d} - \frac{\mu_{1:n}}{2} \norm{d}^2\\
	&= \frac{1}{2 \mu_{1:n}} \norm{\nabla l_n(x_n) - \nabla v_n(x_n)}_*^2,
	\end{align*}
	where the last equality is due to a property of dual norm (e.g. Exercise 3.27 of~\cite{boyd2004convex}).
	Substituting $x_n^\star$ for $x$ and taking the summation over $n$  prove the lemma.	
\end{proof}

Using Lemma~\ref{lm:strong FTL} and Lemma~\ref{lm:regret addon}, we can prove the regret bound of FTL with prediction.

\begin{lemma}[FTL with prediction] \label{lm:regret outer}
	Let $l_n$ be a  $\mu_n$-strongly convex for some $\mu_n \geq 0$.
	In round $n$, let $v_{n+1}$ be a (non)convex function such that $\sum_{m=1}^n w_m l_{m} + w_{m+1} v_{n+1}$ is convex. 
	Suppose the learner plays FTL with prediction, i.e.
	$
	x_{n+1} \in \argmin_{x \in \XX}  \sum_{m=1}^n( w_m l_{m} + w_{m+1} v_{n+1})(x)
	$
	and suppose that  $\sum_{m=1}^{n} w_m \mu_m > 0$. 
	Then 
	\begin{align*}
	\regretw(\XX) \le 
	\sum_{n=1}^{N}\frac{w_n^2 \norm{\nabla l_n(x_n) - \nabla v_n(x_n)}_*^2 }{2 \sum_{m=1}^n w_m \mu_m} 
	\end{align*}
	In particular, if $\mu_n = \mu$, 	$w_n = n^p$, $p \ge 0$, 
	$\regretw(\XX) \le
	\frac{p+1}{2\mu} \sum_{n=1}^N n^{p-1} \norm{\nabla l_n(x_n) - \nabla v_n(x_n)}_*^2
	$.
\end{lemma}
\begin{proof}
	By Lemma~\ref{lm:strong FTL} and Lemma~\ref{lm:regret addon}, we see
	\begin{align*}
	\regretw(\XX) 
	\le \sum_{n=1}^N \paren{l_{1:n}(x_n) - l_{1:n}(x_n^\star)}
	\le \sum_{n=1}^{N}\frac{w_n^2 \norm{\nabla l_n(x_n) - \nabla v_n(x_n)}_*^2 }{2\sum_{m=1}^n w_m \mu_m}.
	\end{align*}
	If $\mu_n = \mu$, 	$w_n = n^p$, and $p \ge 0$, then  it follows from Lemma~\ref{lm:poly sum}
	\begin{align*}
	\regretw(\XX) 
	\le \frac{1}{2\mu} \sum_{n=1}^{N} \frac{n^{2p}}{\frac{n^{p+1}}{p+1}} \norm{\nabla l_n(x_n) -\nabla v_n(x_n)}_*^2 
	=\frac{p+1}{2\mu} \sum_{n=1}^{N} n^{p-1} \norm{\nabla l_n(x_n) -\nabla v_n(x_n) }_*^2.
	\end{align*}
\end{proof}

The next lemma about the regret of FTL is a corollary of Lemma~\ref{lm:regret outer}.
\begin{lemma}[FTL] \label{lm:regret inner}
	Let $l_n$ be $\mu$-strongly convex for some $\mu > 0$.
	Suppose the learner play FTL, i.e. $x_n = \argmin_{x \in \XX }	\sum_{m=1}^{n} w_m l_{m} (x)$. Then 
	$	\regret^w(\XX) \leq \frac{G^2}{2 \mu} \sum_{n=1}^{N} \frac{w_n^2 }{w_{1:n}} $.
	In particular, if $w_n = n^p$, then 	
	\begin{align*}
	\sum_{n=1}^N w_n l_n(x_n) \le 
	\begin{cases}
	\frac{G^2}{2\mu} \frac{p+1}{p} (N+1)^p + \frac{1}{p+1} (N+1)^{p+1}\epsx , 
	&\text{for $p > 0$} \\
	\frac{G^2}{2\mu} \ln (N+1) + N  \epsx,
	&\text{for $p = 0$} \\
	\frac{G^2}{ 2 \mu} O(1)+ \frac{1}{p+1} (N+1)^{p+1}\epsx, 
	&\text{for $-1 <p <0$} \\
	\frac{G^2}{ \mu} +  (\ln N + 1)\epsx,
	&\text{for $p = -1$} \\
	\end{cases}
	\end{align*}
\end{lemma}
\begin{proof}
	By definition of $\regretw (\XX)$,  the absolute cost satisfies
	$
	\sum_{n=1}^N w_n l_n(x_n) = \regretw(\XX) + \min_{x \in\XX} \sum_{n=1}^{N} w_n l_{n} (x).
	$
	We bound the two terms separately.
	For $\regretw(\XX)$, set $v_n = 0$ in Lemma~\ref{lm:regret outer} and we have
	\begin{align*}
	\regretw(\XX) &\leq \frac{G^2}{2 \mu} \sum_{n=1}^{N} \frac{w_n^2 }{w_{1:n}} 
	&\text{(Lemma~\ref{lm:regret outer} and gradient bound)}
	\\
	&= \frac{G^2}{2 \mu}  \sum_{n=1}^N  \frac{n^{2p}}{ \sum_{m=1}^n m^p}  &\text{(Special case $w_n = n^p$)},
	\end{align*}
	in which $\sum_{n=1}^N  \frac{n^{2p}}{ \sum_{m=1}^n m^p}$ is exactly what~\ref{lm:poly special} bounds. 
	On the other hand, the definition of $\epsx$ implies that 
	$
	\min_{x \in\XX} \sum_{n=1}^{N} w_n l_{n} (x)  \le  w_{1:N} \epsx = \sum_{n=1}^N n^p \epsx$, where $ \sum_{n=1}^N n^p$ is bounded by Lemma~\ref{lm:poly sum}.
	Combining these two bounds,
	we conclude the lemma.
\end{proof}